%% file: main-HAL.tex
\documentclass[twoside,11pt]{article}

\usepackage{times}
\usepackage{mathrsfs}
\usepackage{wrapfig}
\usepackage{booktabs}
\usepackage[bbgreekl]{mathbbol}
\usepackage{listings}
\usepackage{float}

\newcommand{\jasa}{0}
\newcommand{\lenghtfig}{12}

\usepackage[abbrvbib, preprint]{jmlr2e}

\usepackage{lastpage}
\jmlrheading{23}{2024}{1-\pageref{LastPage}}{08/24; Revised ---}{---}{21-0000}{Nathan Doumèche, Francis Bach, Eloi Bedek, Gérard Biau, Claire Boyer, Yannig Goude}

\ShortHeadings{Forecasting time-series with constraints}{Doumèche, Bach, Bedek, Biau, Boyer, Goude}
\firstpageno{1}

\usepackage{amsmath}
\numberwithin{theorem}{section}

\begin{document}

\title{Forecasting time series with constraints}

\author{%
 \name Nathan Doumèche \email nathan.doumeche@sorbonne-universite.fr \\
 \addr Sorbonne University, EDF R\&D
 \AND
 \name Francis Bach \email francis.bach@inria.fr \\
 \addr Inria, ENS, PSL Research University
 \AND
 \name Eloi Bedek \email eloi.bedek@edf.fr \\
 \addr EDF R\&D
 \AND
 \name Gérard Biau \email gerard.biau@sorbonne-universite.fr \\
 \addr Sorbonne University, IUF
 \AND
 \name Claire Boyer \email claire.boyer@universite-paris-saclay.fr \\
 \addr Université Paris-Saclay, IUF
 \AND
 \name Yannig Goude \email yannig.goude@edf.fr \\
 \addr EDF R\&D, Paris-Saclay University
}

\editor{---}
\maketitle

\begin{abstract}
    \input{input-abstract}  
\end{abstract}


\input{input-main}

\bibliography{biblio}

\appendix
\input{input-supplementary}

\end{document}

%% file: input-abstract.tex
Time series forecasting presents unique challenges that limit the effectiveness of traditional machine learning algorithms. To address these limitations, various approaches have incorporated linear constraints into learning algorithms, such as generalized additive models and hierarchical forecasting. In this paper, we propose a unified framework for integrating and combining linear constraints in time series forecasting. Within this framework, we show that the exact minimizer of the constrained empirical risk can be computed efficiently using linear algebra alone. This approach allows for highly scalable implementations optimized for GPUs. We validate the proposed methodology through extensive benchmarking on real-world tasks, including electricity demand forecasting and tourism forecasting, achieving state-of-the-art performance. 

%% file: input-main.tex
\section{Introduction}

\paragraph{Time series forecasting.} Time series data are used extensively in many contemporary applications, such as forecasting supply and demand, pricing, macroeconomic indicators, weather, air quality, traffic, migration, and epidemic trends 
\citep{petropoulos2022forecasting}. 
However, regardless of the application domain, forecasting time series presents unique challenges due to inherent data characteristics such as observation correlations, non-stationarity, irregular sampling intervals, and missing values. These challenges limit the availability of relevant data and make it difficult for complex black-box or overparameterized learning architectures to perform effectively, even with rich historical data 
\citep{lim2021time}. 

\paragraph{Constraints in time series.} In this context, many modern frameworks incorporate physical constraints to improve the performance and interpretability of forecasting models. The strongest form of such constraints are typically derived from fundamental physical properties of the time series data and are represented by systems of differential equations. For example, weather forecasting often relies on solutions to the Navier-Stokes equations \citep[][]{schultz2021can}.
In addition to defining physical relationships, differential constraints can also serve as regularization mechanisms. For example, spatiotemporal regression on graphs can involve penalizing the spatial Laplacian of the regression function to enforce smoothness across spatial dimensions \citep[][]{jin2024spatio}.

However, time series rarely satisfy strict differential constraints, often adhering instead to more relaxed forms of constraints \citep[][]{coletta2023on}.
Perhaps the most successful example of such weak constraints are the generalized additive models \citep[GAMs,][]{hastie1986generalized}, which have been applied to time series forecasting in epidemiology \citep{wood2017generalized}, earth sciences \citep{augusting2009modeling}, and energy forecasting \citep{fasiolo2021fast}. 
GAMs model the target time series (or some parameters of its distribution) as a sum of nonlinear effects of the features, thereby constraining the shape of the regression function.
Another example of weak constraint appears in the context of spatiotemporal time series with hierarchical forecasting. Here, the goal is to combine regional forecasts into a global forecast by enforcing that the global forecast must be equal to the sum of the regional forecasts \citep{Wickramasuriya2019optimal}.
Although this may seem like a simple constraint, hierarchical forecasting is challenging because of a trade-off: using more granular regional data increases the available information, but also introduces more noise as compared to the aggregated total. Another common and powerful constraint in time series forecasting arises when combining multiple forecasts \citep{gaillard2014second}. 
This is done by creating a final forecast by weighting each of the initial forecasts, with the constraint that the sum of the weights must equal one.

\paragraph{PIML and time series.} Although  weak constraints have been studied individually and applied to real-world data, a unified and efficient approach is still lacking.
It is important here to mention physics-informed machine learning (PIML), which offers a promising way to integrate constraints into machine learning models. 
Based on the foundational work of \citet{raissi2019PINN}, PIML exploits the idea that constraints can be applied with neural networks and optimized by backpropagation, leading to the development of physics-informed neural networks (PINNs). 
PINNs have been successfully used to predict time series governed by partial differential equations (PDEs) in areas such as weather modeling \citep{kashinath2021physics}, and stiff chemical reactions \citep{ji2021stiff}. 
Weak constraints on the shape of the regression function have also been modeled with PINNs \citep[][]{daw2022lake}. 
However, PINNs often suffer from optimization instabilities and overfitting  \citep{doumeche2023convergence}.  
As a result, alternative methods have been developed for certain differential constraints that offer improved optimization properties over PINNs. For example, data assimilation techniques in weather forecasting have been shown to be consistent with the Navier-Stokes equations \citep{nickl2024on}. 
Moreover, \citet{doumeche2024physicsinformed} showed that forecasting with linear differential constraints can be formulated as a kernel method, yielding closed-form solutions to compute the exact empirical risk minimum. An additional advantage of this kernel modeling is that the learning algorithm can be executed on GPUs, leading to significant speedups compared to the gradient-descent-based optimization of PINNs \citep{doumèche2024physicsinformedkernellearning}.

\paragraph{Contributions.}
In this paper, we present a principled approach to effectively integrate constraints into time series forecasting. Each constrained problem is reformulated as the minimization of an empirical risk consisting of two key components: a data-driven term and a regularization term that enforces the smoothness of the function and the desired physical constraints.
For nonlinear regression tasks, we rely on a Fourier expansion.
Our framework allows for efficient computation of the exact minimizer of the empirical risk, which is easily optimized on GPUs for scalability and performance.

In Section~\ref{sec:weak}, we introduce a unified mathematical framework that connects empirical risks constrained by various forms of physical information. Notably, we highlight the importance of distinguishing between two categories of constraints: shape constraints, which limit the set of admissible functions, and learning constraints, which introduce an initial bias during parameter optimization. In Section~\ref{sec:shape}, we explore shape constraints and illustrate their relevance using the example of electricity demand forecasting.
In Section~\ref{sec:weight}, we define learning constraints and show how they can be applied to tourism forecasting.
This common modeling framework for shape and learning constraints allows for efficient integration of multiple constraints, as illustrated by the WeaKL-T in Section~\ref{sec:weight}, which combines hierarchical forecasting with additive models and transfer learning. 
Each empirical risk can then be minimized on a GPU using linear algebra, ensuring scalability and computational efficiency. 
This direct computation guarantees that the proposed estimator exactly minimizes the empirical risk, preventing convergence to potential local minima---a common limitation of modern iterative and gradient descent methods used in PINNs.
Our method achieves significant performance improvements over state-of-the-art approaches. 
\if\jasa0
{
The code for the numerical experiments and implementation is publicly available at \url{https://github.com/NathanDoumeche/WeaKL}.
}\fi

\section{Incorporating constraints in time series forecasting}
\label{sec:weak}
Throughout the paper, we assume that $n$ observations $(X_{t_1}, Y_{t_1}), \ldots, (X_{t_n}, Y_{t_n})$ are drawn on $\mathbb{R}^{d_1} \times \mathbb{R}^{d_2}$. 
The indices $t_1, \ldots, t_n \in T$ correspond to the times at which an unknown stochastic process $(X,Y):=(X_t, Y_t)_{t \in T}$ is sampled.
Note that, all along the paper, the time steps need not be regularly sampled on the index set $T \subseteq \mathbb R$. 
We focus on supervised learning tasks that aim to estimate an unknown function $f^\star : \mathbb{R}^{d_1} \to \mathbb{R}^{d_2}$, under the assumption that $Y_t = f^\star(X_t) + \varepsilon_t$, where $\varepsilon$ is a random noise term. Without loss of generality, upon rescaling, we assume that $X_t:= (X_{1,t}, \hdots, X_{d_1,t}) \in [-\pi, \pi]^{d_1}$ and $-\pi \leq t_1 \leq  \cdots \leq  t_{n+1}\leq \pi$. The goal is to construct an estimator $\hat{f}$ for $f^\star$. 

A simple example to to keep in mind is when $Y$ is a stationary, regularly sampled time series with $t_j = j/n$, and the lagged value $X_j = Y_{t_{j-1}}$ serves as the only feature. In this specific case, where $d_1 = d_2$, the model simplifies to
$Y_t = f^\star(Y_{t-1/n})+\varepsilon_t$. Thus, the regression setting reduces to an autoregressive model. Of course, we will consider more complex models that go beyond this simple case.

\paragraph{Model parameterization.} We consider parameterized models of the form 
\begin{equation}
    f_{\theta}(X_t) = (f^1_{\theta}(X_t), \hdots, f^{d_2}_{\theta}(X_t)) =  (\langle \phi_1(X_t), \theta_1\rangle, \hdots, \langle \phi_{d_2}(X_t), \theta_{d_2}\rangle),
    \label{eq:model_def}
\end{equation} 
where each component  $f^\ell_\theta(X_t)$ is computed as the inner product of a feature map $\phi_\ell(X_t) \in \mathbb{C}^{D_\ell}$, with $D_\ell \in \mathbb N^\star$, and a vector $\theta_\ell \in \mathbb{C}^{D_\ell}$. 
The parameter vector $\theta \in \mathbb C^{D_1 + \cdots + D_{d_2}}$ of the model is defined as the concatenation of $\theta_1$, $\dots$, $\theta_{d_2}$. 
Note that $f_\theta$ is uniquely determined by $\theta$ and the maps~$\phi_\ell$.
To simplify the notation, we write $\dim(\theta) = D_1 + \cdots + D_{d_2}$. 

Our goal is to learn a parameter $\hat \theta \in \mathbb C^{\dim(\theta)}$ such that $\hat Y_t  = f_{\hat{\theta}}(X_t)$ is an estimator of the target $Y_t$.
Equivalently, $f_{\hat{\theta}}$ is an estimator of the target function $f^\star$. 
To this end, the core principle of our approach is to consider $\hat \theta$ to be a minimizer over $\mathbb C^{\dim(\theta)}$ of an empirical risk of the form
\begin{equation}
    L(\theta) = \frac{1}{n}\sum_{j=1}^n \|\Lambda(f_\theta(X_{t_j})-Y_{t_j})\|_2^2  + \|M\theta\|_2^2,
    \label{eq:risk}
\end{equation}
where $\Lambda$ and $M$ are complex-valued matrices with problem-dependent dimensions, which are not necessarily square. The matrix $M$ encodes a regularization penalty, which may include hyperparameters to be tuned through validation, as we will see in several examples.

\paragraph{Explicit formula for the empirical risk minimizer: WeaKL.}The following proposition shows how to compute the exact minimizer of \eqref{eq:risk}. 
(Throughout the document, $\ast$ denotes the conjugate transpose operation.)
\begin{proposition}[Empirical risk minimizer.] 
\label{prop:emp_risk_min}
Suppose both $M$ and $\Lambda$ are injective.
Then, there is a unique minimizer to \eqref{eq:risk}, which takes the form
\begin{equation}
    \hat \theta = \Big( \Big( \sum_{j=1}^n \mathbb \Phi_{t_j}^\ast \Lambda^\ast \Lambda\mathbb \Phi_{t_j}\Big) + n M^\ast M\Big)^{-1} \sum_{j=1}^n \mathbb \Phi_{t_j}^\ast \Lambda^\ast \Lambda Y_{t_j},
    \label{eq:weakl}
\end{equation}
where $\mathbb \Phi_t$ is the $d_2\times \dim(\theta)$ block-wise diagonal feature matrix at time $t$, defined by
\begin{equation}
\mathbb \Phi_t = \begin{pmatrix}
    \phi_1(X_{t})^\ast & 0& 0 \\
    0 & \ddots & 0 \\
    0 & 0 & \phi_{d_2}(X_{t})^\ast
\end{pmatrix}
\label{eq:feature_matrix}.
\end{equation}
\end{proposition}
This result, proven in Appendix~\ref{proof:kernel}, generalizes well-known results on kernel ridge regression \citep[see, e.g.,][Equation 10.17]{mehri2012foundations}. 
In the rest of the paper, we refer to the estimator $\hat \theta$ as the weak kernel learner (WeaKL). The strength of WeaKL lies in its exact computation via~\eqref{eq:weakl}. Unlike current implementations of GAMs and PINNs, WeaKL is free from optimization errors. Furthermore, since WeaKL relies solely on linear algebra, it can take advantage of GPU programming to accelerate the learning process. 
This efficiency enables effective hyperparameter optimization, as demonstrated in Section~\ref{sec:energy_crisis} through applications to electricity demand forecasting.

\paragraph{Algorithmic complexity.} The formula \eqref{eq:weakl} used in this article to minimize the empirical risk \eqref{eq:risk} can be implemented with a  complexity of $ O(\dim(\theta)^3 +  \dim(\theta)^2 n)$. 
Note that the dimensions $d_1$ and $d_2$ of the problem only impact the complexity of WeaKL through $\dim(\theta) = D_1 + \cdots + D_{d_2}$. 
By construction, $\dim(\theta) \geq d_2$, but the influence of $d_1$ is more subtle and depends on the chosen dimension $D_\ell$ of the maps $\phi_j: [-\pi, \pi]^{d_1}\to \mathbb{C}^{D_j}$. 
In particular, if all the maps have the same dimension, i.e., $D_\ell = D$, then $\dim(\theta) = Dd_2$.

Notably, this implementation runs in less than ten seconds on a standard GPU (e.g., an NVIDIA $L4$ with $24$ GB of RAM) when $\dim(\theta) \leq 10^3$ and $n \leq 10^5$. 
We believe that this framework is particularly well suited for time series, where data sampling is often costly, thus limiting both $n$ and $d_2$. Moreover, in many cases, the distribution of the target time series changes significantly over time, making only the most recent observations relevant for forecasting. This further limits the size of $n$. For example, in the Monash time series forecasting archive \citep{Godahewa2021Monash}, $19$ out of $30$ time series have $d_2 \leq 10^3$ and $n \leq 10^5$. 
However, there are relevant time series where either the dimension $d_2$ or the number of data points $n$ is large. 
In such cases, finding an exact minimizer of the empirical risk \eqref{eq:risk} becomes very computationally expensive. 
Efficient techniques have been developed to approximate the minimizer of \eqref{eq:risk} in these regimes \citep[see, e.g.,][]{meanti2020kernel}, but a detailed discussion of these methods is beyond the scope of this paper.

\paragraph{Some important examples.}
Let us illustrate the mechanism with two fundamental examples. Of course, the case where $\phi_\ell(x) = x$ and where $\Lambda$ and $M$ are identity matrices corresponds to the well-known ridge linear regression. 
On the other hand, a powerful example of a nonparametric regression map is the Fourier map, defined as $\phi_\ell(x) = (\exp(i \langle x, k \rangle / 2))_{\|k\|_\infty \leq m}^\top = (\exp(i \langle x, k \rangle / 2))_{-m\leq k_1, \hdots, k_{d_1} \leq m}^\top$, where the Fourier frequencies are truncated at $m \geq 0$. 
This map leverages the expressiveness of the Fourier basis to capture complex patterns in the data. Thus, for the $\ell$-th component of $f_{\theta}$, we consider the Fourier decomposition
\[
f^\ell_{ \theta}(x) =  \sum_{\|k\|_\infty \leq m}  \theta_{\ell,k} \exp(-i \langle x, k\rangle/2),
\]
which can approximate any function in $L^2([-\pi, \pi]^{d_1}, \mathbb{R})$ as $m \to \infty$. In this example, we have $\theta_{\ell}=(\theta_{\ell,k})_{\|k\|_\infty \leq m}^\top \in \mathbb C^{(2m+1)^d}$. 
Next, for $s \in \mathbb N^\star$, 
let $M$ be the $(2m+1)^{d_1}\times (2m+1)^{d_1}$ positive diagonal matrix such that
\[
\|M \theta_\ell\|_2^2  = \lambda \sum_{\|k\|_{\infty} \leq m} \theta_{\ell,k}^2 (1+\|k\|_2^{2s}),
\]
where $\lambda > 0$ is an hyperparameter.
Then, $\|M \theta_\ell\|_2$
is a Sobolev norm on the derivatives up to order $s$ of $f_{\theta_\ell}$.
When $\lambda = 1$, we will denote this norm by $\|f_{\theta}^\ell\|_{H^s}$. 
This approach regularizes the smoothness of $f_{\hat{\theta}}^{\ell}$, encouraging the recovery of smooth solutions. 
Moreover, choosing $\Lambda$ as the identity matrix and $\lambda = n^{-2s/(2s+d_1)}$ achieves the Sobolev minimax rate $\mathbb E(\|f_{\hat \theta}^\ell(X) -Y_\ell\|_2^2) = O(n^{-2s/(2s+d_1)})$ \citep{blanchard2020kernel}. 
This result justifies why the Fourier decomposition serves as an effective nonparametric mapping. 

These fundamental examples illustrate the richness of the approach, making it possible to incorporate constraints into models of chosen complexity, from very light models like linear regression, up to nonparametric models such as Fourier maps.

\paragraph{Classification of the constraints.} In order to clarify our discussion as much as possible, we find it helpful, after a thorough analysis of the existing literature, to consider two families of constraints. This distinction arises from the need to address two fundamentally different aspects of the forecasting problem.
\begin{enumerate}
\item {\bf Shape constraints}, described in Section~\ref{sec:shape}, include additive models, online adaption after a break, and forecast combinations (detailed in Appendix~\ref{sec:combination}). In these models, prior information is incorporated by selecting custom maps $\phi_\ell$. The set of admissible models  $f_\theta$ is thus restricted by shaping the structure of the function space through this choice of maps. Here, the matrix $M$ serves only as a regularization term, while $\Lambda$ is the identity matrix.

\item {\bf Learning constraints}, described in Section~\ref{sec:weight}, include transfer learning, hierarchical forecasting, and differential constraints (detailed in Appendix~\ref{sec:diff}). In these models, prior information or constraints are incorporated through the matrices $M$ and $\Lambda$. The goal is to increase the efficiency of parameter learning by introducing additional regularization.
\end{enumerate}
It is worth noting, however, that certain specific shape constraints cannot be penalized by a kernel norm, such as those in isotonic regression. In the conclusion, we discuss possible extensions to account for such constraints.

\section{Shape constraints}
\label{sec:shape}
\subsection{Mathematical formulation}
In this section, we introduce relevant feature maps $\phi$ that incorporate prior knowledge about the shape of the function  $f^\star:[-\pi,\pi]^{d_1}\to \mathbb{C}^{d_2}$. To simplify the notation, we focus on the one-dimensional case where $d_2 = 1$ and $\Lambda = 1$. 
This simplification comes without loss of generality, since the feature maps developed in this section can be applied directly to \eqref{eq:model_def}.

As a result, the model reduces to $f_{\theta}(X_t) = \langle \phi_1(X_t), \theta_1 \rangle$, and \eqref{eq:weakl} simplifies to
\begin{equation}
    \hat \theta = ( \mathbb \Phi^\ast \mathbb \Phi + n M^\ast M)^{-1}  \mathbb \Phi^\ast \mathbb Y,
    \label{eq:weakl2}
\end{equation}
where $\mathbb Y = (Y_{t_1}, \hdots, Y_{t_n})^\top \in \mathbb R^n$ and the $n\times \dim(\theta)$ matrix $\mathbb \Phi$ takes the form 
\[ \mathbb \Phi = (\phi_1(X_{t_1})\mid \cdots \mid \phi_1(X_{t_n}))^\ast.
\]
Note that $\mathbb \Phi$ is the classical feature matrix, and that it is related to the matrix $\mathbb \Phi_t$ of \eqref{eq:feature_matrix} by $\mathbb \Phi^\ast \mathbb \Phi = \sum_{j=1}^n\mathbb \Phi_{t_j}^\ast \mathbb \Phi_{t_j} = \sum_{j=1}^n \phi_1(X_{t_j}) \phi_1(X_{t_j})^\ast$.

\paragraph{Additive model: Additive WeaKL.} The additive model constraint assumes that $f^\star(x_1, \hdots, x_{d_1}) = \sum_{\ell=1}^{d_1} g_\ell^\star(x_\ell)$, where $g_\ell^\star: \mathbb{R} \to \mathbb{R}$ are univariate functions. This constraint is widely used in data science, both in classical statistical models \citep{hastie1986generalized} and in modern neural network architectures \citep{agarwal2021neural}. Indeed, additive models are interpretable because the effect of each feature $x_\ell$ is captured by its corresponding function $g_\ell^\star$. In addition, univariate effects are easier to estimate than multivariate effects \citep{Ravikumar2009sparse}. These properties allow the development of efficient variable selection methods \citep[see, for example,][]{marra2011practical}, similar to those used in linear regression.

In our framework, the additivity constraint directly translates into the model as
\[
f_{\theta}(X_t) =  \langle \phi_{1}(X_{t}), \theta_{1}\rangle = \langle \phi_{1,1}(X_{1,t}), \theta_{1,1}\rangle + \cdots + \langle \phi_{1,d_1}(X_{d_1,t}), \theta_{1,d_1}\rangle,
\]
where $\phi_1$ is the concatenation of the maps $\phi_{1,\ell}$, and $\theta_1$ is the concatenation of the vectors  $\theta_{1,\ell}$. 
Note that the maps $\phi_{1,\ell}$ and the vectors $\theta_{1, \ell}$ can be multidimensional, depending on the model.
In this formulation, the effect of each feature is modeled by the function $g_\ell(x_\ell) = \langle \phi_{1,\ell}(x_\ell), \theta_{1,\ell}\rangle$, which can be either linear or nonlinear in $x_\ell$.
The empirical risk then takes the form
\begin{equation}
    L(\theta) = \frac{1}{n} \sum_{j=1}^n |f_\theta(X_{t_j}) - Y_{t_j}|^2 + \sum_{\ell=1}^{d_1}\lambda_\ell\|M_\ell\theta_{1,\ell}\|_2^2, \label{eq:weaklGAM}
\end{equation}
where $\lambda_\ell >0$ are hyperparameters and $M_\ell$ are regularization matrices.
There are three types of effects that can be taken into account:
\begin{itemize}
    \item[$(i)$] A linear effect is obtained by setting $\phi_{1,\ell}(x_\ell) = x_\ell \in \mathbb R$. 
    To regularize the parameter $\theta_{1, \ell}$, we set $M_\ell = 1$. This corresponds to a ridge penalty.
    \item[$(ii)$] A nonlinear effect can be modeled using the Fourier map $\phi_{1,\ell}(x_\ell) = (\exp(i  k x_\ell  / 2))_{-m\leq k \leq m}^\top$. 
    To regularize the parameter $\theta_{1, \ell}$, we set $M_\ell$ to be the $(2m+1)\times (2m+1)$ diagonal matrix defined by $M_\ell =\mathrm{Diag}((\sqrt{1+k^{2s}})_{-m\leq k\leq m})$, penalizing the Sobolev norm. 
    A common choice for the smoothing parameter $s$, as used in GAMs, is $s = 2$ \citep[see, e.g.,][]{wood2017generalizedbook}.
    \item[$(iii)$] If $x_\ell$ is a categorical feature, i.e., $x_\ell$ takes values in a finite set $E$, we can define a bijection $\psi: E \to \{1, \hdots, |E|\}$. The effect of $x_\ell$ can then be modeled as $g_\ell(x_\ell) = \langle \phi_{1,\ell}(x_\ell), \theta_1 \rangle$, where $\phi_\ell = \phi \circ \psi$ and $\phi$ is the Fourier map with $m = \lfloor |E|/2 \rfloor$. To regularize the parameter $\theta_{1, \ell}$, we set $M_\ell$ as the identity matrix, which corresponds to applying a ridge penalty.
\end{itemize}
Overall, similar to GAMs, WeaKL can be optimized to fit additive models with both linear and nonlinear effects. The parameter $\hat \theta$ of the WeaKL can then be computed using \eqref{eq:weakl2}
with 
\[M = \begin{pmatrix}
        \sqrt{\lambda_1} M_1& 0  & 0\\
        0 & \ddots&  0\\
        0 & 0& \sqrt{\lambda_{d_1}} M_{d_1}
    \end{pmatrix}.\]
To stress that this WeaKL results from the enforcement of additive constraints, we call it the \textit{additive WeaKL}.
Note that, contrary to GAMs where identifiability issues must be addressed \citep{wood2017generalizedbook}, WeaKL does not require further regularization, since $\hat \theta$ is the unique minimizer of the empirical risk~$L$. 
Note that the hyperparameters $\lambda_\ell$, along with the number $m$ of Fourier modes and the choice of feature maps  $\phi_\ell$, can be determined by model selection, as described in Appendix~\ref{sec:tuning}.

\paragraph{Online adaption after a break: Online WeaKL.}
For many time series, the dependence of $Y$ on $X$ can vary over time. For example, the behavior of $Y$ may change rapidly following extreme events, resulting in structural breaks. A notable example is the shift in electricity demand during the COVID-19 lockdowns, as illustrated in use case $1$. To provide a clear mathematical framework, we assume that the distribution of $(X, Y)$ follows an additive model that evolves smoothly over time. Formally, considering $(t, X_t)$ as a feature vector, we assume that
\begin{equation}
    f^\star(t, x_1, \hdots, x_{d_1}) = h_0^\star(t)+ \sum_{\ell=1}^{d_1} (1+ h_\ell^\star(t))  g_\ell^\star(x_\ell),
    \label{eq:model}
\end{equation}
where $g_\ell^\star$ and $h_\ell^\star$ are univariate functions. This model forms the core of the Kalman-Viking algorithm \citep{vilmarest2024viking}, which has demonstrated state-of-the-art performance in forecasting electricity demand and renewable energy production \citep{obst2021adaptive, vilmarest2022state, Vilamarest2024adaptive}. 

We assume that we have at hand estimators $\hat g_\ell$ of $g_\ell^\star$ that we want to update over time. For example, these estimators can be obtained by fitting an additive WeaKL model, initially assuming $h_\ell^\star = 0$. The functions $h_\ell^\star$ are then estimated by minimizing the empirical risk
\begin{equation}
    L(\theta) = \frac{1}{n}\sum_{j=1}^n \Big|h_{\theta_0}(t_j) + \sum_{\ell=1}^{d_1} (1+h_{\theta_\ell}(t_j)) \hat g_\ell(X_{\ell,t_j})-Y_{t_j}\Big|^2 + \sum_{0\leq \ell \leq d_1} \lambda_\ell\|h_{\theta_\ell}\|_{H^s}^2,
    \label{eq:risk_online}
\end{equation}
where $\lambda_\ell > 0$ are hyperparameters regularizing the smoothness of the functions $h_{\theta_\ell}$. Here, $h_{\theta}(t) = \langle \phi(t), \theta\rangle$, and $\phi$ is the Fourier map $\phi(t) =(\exp(i k t/2))_{-m\leq k \leq m}^\top$. The prior $h_{\theta_\ell} \simeq 0$ reflects the idea that the best a priori estimate of $Y$'s behavior follows the stable additive model. Defining $W_t = Y_t - \sum_{\ell=1}^{d_1}\hat g_\ell(X_{\ell,t})$, the empirical risk can be reformulated as
\begin{equation*}
    L(\theta) = \frac{1}{n}\sum_{j=1}^n |\langle \phi_1(t_j, X_{t_j}), \theta\rangle - W_{t_j}|^2 + \|M \theta\|_2^2,
\end{equation*}
where
$\phi_1(t, X_t) = 
    ((\exp(ik t/2))_{- m\leq k \leq  m},
    (\hat g_\ell(X_{\ell,t})\exp(ik t/2))_{- m\leq k \leq  m})_{\ell=1}^{d_1})^\top \in \mathbb C^{(2m+1)(d_1+1)}$,
\[M = \begin{pmatrix}
        \sqrt{\lambda_0} D& 0  & 0\\
        0 & \ddots&  0\\
        0 & 0& \sqrt{\lambda_{d_1}} D
    \end{pmatrix},\]
and $D$ is the $(2m+1)\times (2m+1)$ diagonal matrix
$D =\mathrm{Diag}((\sqrt{1+k^{2s}})_{-m\leq k\leq m})$.
From \eqref{eq:weakl2}, we deduce that the unique minimizer of the empirical loss $L$ is
\begin{equation}
    \hat \theta  = ({\mathbb{\Phi}} ^\ast {\mathbb{\Phi}} + n  M^\ast M)^{-1}{\mathbb \Phi}^\ast  \mathbb W,
    \label{eq:online}
\end{equation}
where  $\mathbb W = (W_{t_1}, \hdots, W_{t_n})^\top \in \mathbb R^n$. 

This formulation allows to forecast the time series $Y$ at the next time step, $t_{n+1}$, using
\begin{align*}
\hat Y_{t_{n+1}} &= f_{\hat \theta}(t_{n+1}, X_{t_{n+1}}) = \langle \phi_1(t_{n+1}, X_{t_{n+1}}), \hat \theta\rangle \\
&=   h_{\hat \theta_0}(t_{n+1}) + \sum_{\ell=1}^{d_1} (1+h_{\hat \theta_\ell}(t_{n+1})) \hat g_\ell(X_{\ell,t_{n+1}}).
\end{align*}
Since the model is continuously updated over time, this corresponds to an online learning setting.
To emphasize that Equation~\eqref{eq:online} arises from an online adaptation process, we refer to this model as the \textit{online WeaKL}.
Unlike the Viking algorithm of \citet{vilmarest2024viking}, which approximates the minimizer of the empirical risk through an iterative process, online WeaKL offers a closed-form solution and exploits GPU parallelization for significant speedups.
As shown in Section~\ref{sec:energy_crisis}, our approach leads to improved performance in electricity demand forecasting.

\subsection{Application to electricity load forecasting}
\label{sec:energy_crisis}
In this subsection, we apply shape constraints to two use cases in electricity demand forecasting and demonstrate the effectiveness of our approach.
In these electricity demand forecasting problems, the focus is on short-term forecasting, with particular emphasis on the recent non-stationarities caused by the COVID-19 lockdowns and by the energy crisis.

\paragraph{Electricity load forecasting and non-stationarity.} Accurate demand forecasting is critical due to the costly nature of electricity storage, coupled with the need for supply to continuously match demand.  
Short-term load forecasting, especially for 24-hour horizons, is particularly valuable for making operational decisions in both the power industry and electricity markets.
Although the cost of forecasting errors is difficult to quantify, a $1\%$ reduction in error is estimated to save utilities several hundred thousand USD per gigawatt of peak demand \citep{hong2016probabilistic}. 
Recent events such as the COVID-19 shutdown have significantly affected electricity demand, highlighting the need for updated forecasting models  \citep{zarbakhsh2022human}.

\paragraph{Use case 1: Load forecasting during COVID.} In this first use case, we test the performance of our WeaKL on the IEEE DataPort Competition on Day-Ahead Electricity
Load Forecasting \citep{Farrokhabadi2022day}.
Here, the goal is to forecast the electricity demand of an unknown country during the period following the Covid-19 lockdown.
The winning model of this competition was the Viking model of Team~4 \citep{vilmarest2022state}, with a mean absolute error (MAE) of $10.9$ gigawatts (GW). For comparison, a direct translation of their model into the online WeaKL framework---using the same features and maintaining the same additive effects---results in an MAE of $10.5$ GW. In parallel, we also apply the online WeaKL methodology without relying on the variables selected by \citet{vilmarest2022state}. Instead, we determine the optimal hyperparameters $\lambda_\ell$ and select the feature maps $\phi_\ell$ through a hyperparameter tuning process (see Appendix~\ref{sec:tuning}). This leads to a different selected model with a MAE of $9.9$ GW (see Appendix~\ref{sec:case_study1} for a complete description of the models). 
Thus, the online WeaKL given by \eqref{eq:online} outperforms the state-of-the-art by $9\%$. 
As done in the IEEE competition \citep{Farrokhabadi2022day}, we assess the significance of this result by evaluating the MAE skill score using a block bootstrap approach (see Appendix~\ref{sec:case_study1}). 
It shows that the online WeaKL outperforms the winning model proposed by \citet{vilmarest2022state} with a probability above $90\%$.
The updated results of the competition are presented in Table~\ref{tab:ieee}. 
Note that a great variety of models were benchmarked in this competition, like Kalman filters (Team~4), autoregressive models (Teams~4 and 7), random forests (Teams~4 and 6), gradient boosting (Teams~6 and 36),  deep residual networks (Team~19), and averaging (Team~13).

\begin{table}[h]
\centering
\caption{Performance of the online WeaKL and of the top $10$ participants of the IEEE competiton. A specific bootstrap test shows that the WeaKL significantly outperform the winning team.}
\begin{tabular}{lccccccccccc}
\toprule
Team & WeaKL &  4 &   14 &   7 &   36 &   19 &   23 &   9 &   25 &   13 &   26 \\
\midrule
MAE (GW)  & \textbf{9.9} & 10.9 & 11.8 & 11.9 & 12.3 & 12.3 & 13.9 & 14.2 & 14.3 & 14.6 & 15.4\\
\bottomrule
\end{tabular}

\label{tab:ieee}
\end{table}

\paragraph{Use case 2: Load forecasting during the energy crisis.}
In this second use case, we evaluate the performance of our WeaKL within the open source benchmark framework proposed by \citet{doumeche2023human}. This benchmark provides a comprehensive evaluation of electricity demand forecasting models, incorporating the GAM boosting model of \cite{bentaieb2014a}, the GAM of \cite{obst2021adaptive}, the Kalman models of \cite{vilmarest2022state}, the time series random forests of \cite{gohery2023random}, and the Viking model of \cite{Vilamarest2024adaptive}. 
The goal here is to forecast the French electricity demand during the energy crisis in the winter of 2022-2023. Following the war in Ukraine and maintenance problems at nuclear power plants, electricity prices reached an all-time high at the end of the summer of 2022. In this context, French electricity demand decreased by $10\%$ compared to its historical trends \citep{doumeche2023human}. 
This significant shift in electricity demand can be interpreted as a structural break, which justifies the application of the online WeaKL given by \eqref{eq:online}.

In this benchmark, the models are trained from 8 January 2013 to 1 September 2022, and then evaluated from 1 September 2022 to 28 February 2023.
The dataset consists of temperature data from the French meteorological administration \citet{meteoFrance}, and electricity demand data from the French  transmission system operator \citet{rteData}, sampled with a half-hour resolution.  
This translates into the feature variable 
\[X =(\mathrm{Load}_1, \mathrm{Load}_7, \mathrm{Temp}, \mathrm{Temp}_{950},  \mathrm{Temp}_{\mathrm{max 950}}, \mathrm{Temp}_{\mathrm{min 950}}, \mathrm{ToY},  \mathrm{DoW}, \mathrm{Holiday},t),
\]
where $\mathrm{Load}_1$ and $\mathrm{Load}_7$ are the electricity demand lagged by one day and seven days, $\mathrm{Temp}$ is the temperature, and $\mathrm{Temp}_{950}$,  $\mathrm{Temp}_{\mathrm{max 950}}$, and $\mathrm{Temp}_{\mathrm{min 950}}$ are smoothed versions of $\mathrm{Temp}$. The time of year $\mathrm{ToY} \in \{1, \hdots, 365\}$ encodes the position within the year. 
The day of the week $\mathrm{DoW} \in \{1, \hdots, 7\}$ encodes the position within the week. 
In addition, $\mathrm{Holiday}$ is a boolean variable set to one during holidays, and $t$ is the timestamp. 
Here, the target $Y = \mathrm{Load}$ is the electricity demand, so $d_1 = 10$ and $d_2 = 1$.

We compare the performance of two of our WeaKLs against this benchmark.
First, our additive WeaKL is a direct translation of the GAM formula proposed by  \cite{obst2021adaptive} into the additive WeaKL framework given by \eqref{eq:weaklGAM}. Thus, $f_\theta(x) = \sum_{\ell=1}^{10} g_\ell(x_\ell)$, where:
\begin{itemize}
    \item the effects $g_1$, $g_2$, and $g_{10}$ of $\mathrm{Load}_1$, $\mathrm{Load}_7$, and $t$ are linear,
    \item the effects $g_3,\dots, g_7$ of $\mathrm{Temp}$, $\mathrm{Temp}_{950}$,  $\mathrm{Temp}_{\mathrm{max 950}}$, $ \mathrm{Temp}_{\mathrm{min 950}}$, and $\mathrm{ToY}$ are nonlinear with $m=10$,
    \item the effects $g_8$ and $g_9$ of   $\mathrm{DoW}$ and $\mathrm{Holiday}$ are categorical with $|E| = 7$ and $|E| = 2$.
\end{itemize}

\begin{wrapfigure}[\lenghtfig]{r}{0.4\textwidth}
    \centering 
    \vspace{-1em}
    \includegraphics[width=\linewidth, trim={0.4cm 0.3cm 0.2cm 0.9cm},clip]{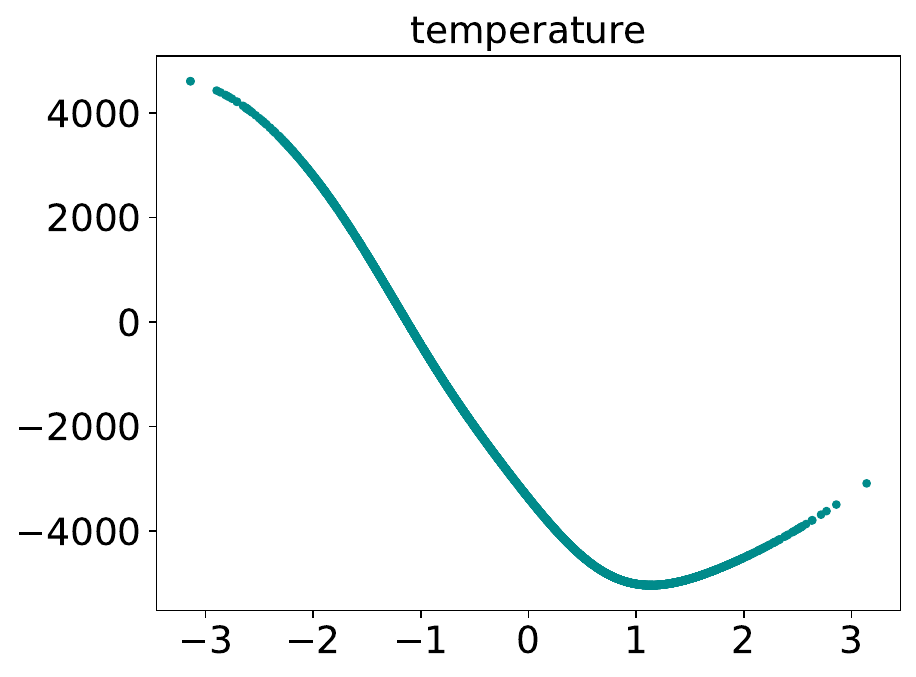}
    \vspace{-2em}
    \caption{Effect in MW of the temperature in the additive WeaKL.}
    \label{fig:WeaKL_effect}
 \end{wrapfigure}
\noindent The weights $\theta$ are learned using data from 2013 to 2021, while the optimal hyperparameters $\lambda_1, \dots, \lambda_{10}$ are tuned using a validation set covering the period from 2021 to 2022.
Once the additive WeaKL is learned, it becomes straightforward to interpret the impact of each feature on the model. For example, the effect $\hat g_3: \mathrm{Temp} \mapsto \langle\phi_{1,3}(\mathrm{Temp}), \hat \theta_{1, 3}\rangle $ of the rescaled temperature feature ($\mathrm{Temp} \in [-\pi, \pi]$) is illustrated in Figure~\ref{fig:WeaKL_effect}.

Second, our online WeaKL is the online adaptation of $f_\theta$  in response to a structural break, as described by \eqref{eq:online}.
The hyperparameters $\lambda_0, \dots, \lambda_{10}$ in \eqref{eq:risk_online} are chosen to minimize the error over a validation period from $1$ April $2020$ to $1$ June $2020$, corresponding to the first COVID-19 lockdown. Note that this validation period does not immediately precede the test period, which is uncommon in time series analysis. However, this choice ensures that the validation period contains a structural break, making it as similar as possible to the test period. Next, the functions $h_0, \dots, h_{10}$ in \eqref{eq:model} are trained on a period starting from $1$ July $2020$, and updated online. 

The results are summarized in Table~\ref{table_score_target_agg2}. The errors and their standard deviations are assessed by stationary block bootstrap (see Appendix~\ref{sec:block-bootstrap}). Since holidays are notoriously difficult to predict, performance is evaluated over the entire period (referred to as \textit{Including holidays}), and separately excluding holidays and the days immediately before and after (referred to as \textit{Excluding holidays}). 
Over both test periods, the additive WeaKL  significantly outperforms the GAM, while the online WeaKL outperforms the state-of-the-art by more than $10\%$ across all metrics.

Figure~\ref{fig:err_time_weaKL} shows the errors of the WeaKLs as a function of time during the test period, which includes holidays. 
During the sobriety period, electricity demand decreased, causing the additive WeaKL to overestimate demand, resulting in a negative bias. Interestingly, this bias is effectively corrected by the online WeaKL, which explains its strong performance. This shows that the online update of the effects effectively corrects biases caused by shifts in the data distribution.

Then, we compare the running time of the algorithms. 
Note that, during hyperparameter tuning, the GPU implementation of WeaKL makes it possible to train $1.6\times 10^5$ additive WeaKL over a period of eight years in less than five minutes on a single standard GPU (NVIDIA $L4$). 
As for the online WeaKL, the training is more computationally intensive because the model must be updated in an online fashion.
However, training $9.2 \times 10^3$ online WeaKLs over a period of two years takes less than two minutes.
This approach is faster than the Viking algorithm, which takes over $45$ minutes to evaluate the same number of parameter sets on the same dataset, even when using $10$ CPUs in parallel. A detailed comparison of the running times for all algorithms is provided in Appendix~\ref{sec:sobriety}.
\begin{table}[H]
\centering
\caption{Benchmark for load forecasting during the energy crisis}
\begin{tabular*}{\textwidth}{@{\extracolsep\fill}lcccc}
  \toprule
  & \multicolumn{2}{@{}c@{}}{ Including holidays} & \multicolumn{2}{@{}c@{}}{Excluding holidays} \\\cmidrule{2-3}\cmidrule{4-5}%
 & RMSE (GW)& MAPE (\%) &  RMSE (GW)& MAPE (\%)\\
  \midrule
  \textit{Statistical model} &&&&\\
  Persistence (1 day) & 4.0$\pm$0.2 & 5.5$\pm$0.3 & 4.0$\pm$0.2  & 5.0$\pm$0.3\\
  SARIMA  &  2.4$\pm$0.2   & 3.1$\pm$0.2 & 2.0$\pm$0.2  & 2.6$\pm$0.2\\
  GAM & 2.3$\pm$0.1 & 3.5$\pm$0.2   & 1.70$\pm$0.06 & 2.6$\pm$0.1 \\
  \midrule
    \textit{Data assimilation }\\
  Static Kalman & 2.1$\pm$0.1 & 3.1$\pm$0.2   &  1.43$\pm$0.05 & 2.20$\pm$0.08 \\
  Dynamic Kalman & 1.4$\pm$0.1 & 1.9$\pm$0.1   & 1.10$\pm$0.04 & 1.58$\pm$0.05  \\
    Viking & 1.5$\pm$0.1 & 1.8$\pm$0.1 &  0.98$\pm$0.04 & 1.33$\pm$0.04\\
    Aggregation & 1.4$\pm$0.1 & 1.8$\pm$0.1 & 0.96$\pm$0.04 & 1.36$\pm$0.04\\
    \midrule
    \textit{Machine learning}\\
    GAM boosting & 2.6$\pm$0.2 & 3.7$\pm$0.2 & 2.3$\pm$0.1 & 3.3$\pm$0.2 \\
    Random forests &  2.5$\pm   $0.2& 3.5$\pm$0.2& 2.1$\pm$0.1 & 3.0$\pm$0.1\\
    Random forests + bootstrap & 2.2$\pm$0.2 & 3.0$\pm$0.2 & 1.9$\pm$0.1 & 2.6$\pm$0.1\\
    \midrule
    \textit{WeaKLs}\\
    Additive WeaKL & 1.95$\pm$0.08 & 3.0 $\pm$0.1& 1.55$\pm$0.06 & 2.32$\pm$0.09  \\
    Online WeaKL &  \textbf{1.14$\pm$0.09}& \textbf{1.5$\pm$0.1}&  \textbf{0.87$\pm$0.04 }& \textbf{1.17$\pm$0.05} \\
   \bottomrule
\end{tabular*}
\label{table_score_target_agg2}
\end{table}
Both use cases demonstrate that WeaKL models are very powerful. Not only are they highly interpretable---thanks to their ability to fit into a common framework and produce simple formulas---but they are also competitive with state-of-the-art techniques in terms of both optimization efficiency (they can run on GPUs) and performance (measured by MAPE and RMSE).
\begin{figure}
    \centering
    \includegraphics[width=1\linewidth, trim={2.2cm 1.2cm 3cm 2.1cm},clip]{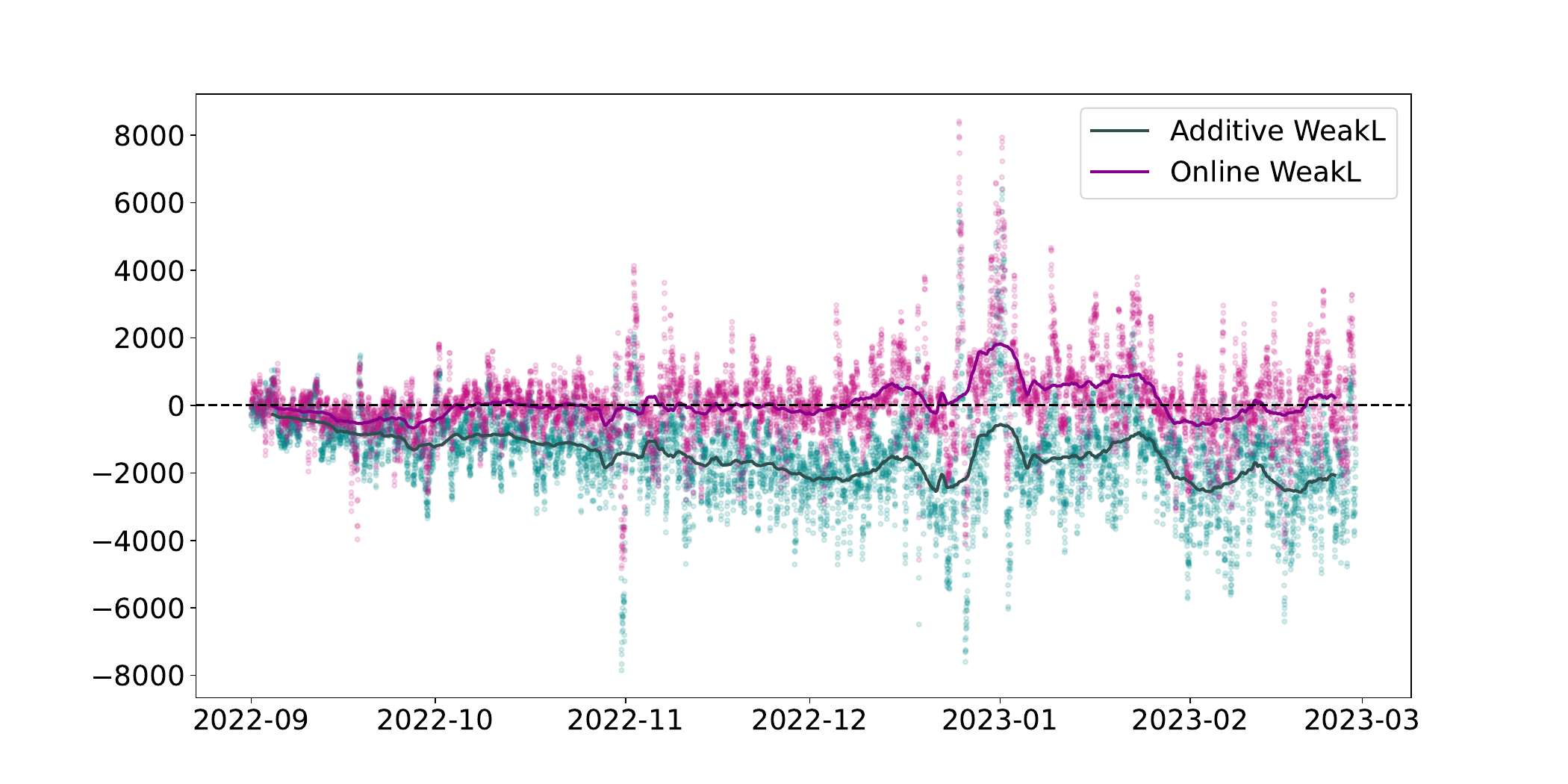}
    \caption{Error $Y_t - \hat Y_t$ in MW of the WeaKLs on the test period including holidays. Dots represent individual observations, while the bold curves indicate the one-week moving averages.}
    \label{fig:err_time_weaKL}
\end{figure}
\section{Learning constraints}
\label{sec:weight}
\subsection{Mathematical formulation}
Section~\ref{sec:shape} focused on imposing constraints on the shape of the regression function $f^\star$. In contrast, the goal of the present section is to impose constraints on the parameter $\theta$. 
We begin with a ge\-ne\-ral method to enforce linear constraints on $\theta$, and subsequently apply this framework to transfer learning, hierarchical forecasting, and differential constraints.

\paragraph{Linear constraints.}
Here, we assume that $f^\star$ satisfies a linear constraint.
By construction of $f_\theta$ in \eqref{eq:model_def}, such a linear constraint directly translates into a constraint on $\theta$.
For example, the linear constraint $f^{\star\,1}(X_t) = 2f^{\star\,2}(X_t)$ can be implemented by enforcing $\theta_1 = 2\theta_2$.
Thus, in the following, we assume a prior on $\theta$ in the form of a linear constraint. Formally, we want to enforce that $\theta \in \mathcal{S}$, where $\mathcal{S}$ is a known linear subspace of $\mathbb{C}^{\dim(\theta)}$.
Given an injective $\dim(\theta) \times \dim(\mathcal{S})$ matrix $P$ such that $\mathrm{Im}(P) = \mathcal{S}$, then, as shown in Lemma~\ref{lem:ortho}, $\|C\theta\|_2^2$ is the square of the Euclidean distance between $\theta$ and $\mathcal S$, where $C=\mathrm{I}_{\dim(\theta)} - P(P^\ast P)^{-1}P^\ast$.  In particular, $\|C\theta\|_2^2 = 0 $ is equivalent to $\theta \in \mathcal S$, and $\|C\theta\|_2^2 = \|\theta\|_2^2$ if $\theta \in \mathcal S^\perp$. From this observation, there are two ways to enforce $\theta \in \mathcal S$ in the empirical risk \eqref{eq:risk}.

On the one hand, suppose that $f^\star$ exactly satisfies the linear constraint.
This happens in particular when the constraint results from a physical law. 
For example, to build upon the use cases of Section~\ref{sec:energy_crisis}, assume that we want to forecast the electricity load of different regions of France, i.e., the target $Y \in \mathbb R^3$ is such that $Y_1$ is the load of southern France, $Y_2$ is the load of northern France, and $Y_3 = Y_1+Y_2$ is the national load. 
This prototypical example of hierarchical forecasting is presented in Section~\ref{sec:toy-example}, where we show how incorporating even a simple constraint can significantly improve the model's performance. In this example, we know that $f^\star$ satisfies the constraint $f^{\star\,3} = f^{\star\,1} + f^{\star\,2}$.
When dealing with such exact priors, a sound approach is to consider only parameters $\theta$ such that $C\theta = 0$, or equivalently, $\theta = P\theta'$. Letting $\Pi_\ell$ be the $D_\ell\times \dim(\theta)$ projection matrix such that $\theta_\ell = \Pi_\ell \theta$, we have $\langle \phi_\ell (X_t), \theta_\ell\rangle = \langle \phi_\ell (X_t), \Pi_\ell \theta\rangle  = \langle P^\ast \Pi_\ell^\ast \phi_\ell (X_t),  \theta'\rangle$. 
Thus, minimizing the empirical risk \eqref{eq:risk} over $\theta' \in \mathbb C^{\dim(\mathcal S)}$ simply requires changing $\phi_\ell$ to $P^\ast \Pi_\ell^\ast \phi_\ell$, which is equivalent to replacing $\mathbb{\Phi}_t$ with $\mathbb{\Phi}_t P$ in \eqref{eq:weakl}. 

On the other hand, suppose that the linear constraint serves as a good but inexact prior. 
For example, building on the last example, let $X_t$ be the average temperature in France at time $t$. We expect the loads $Y_1$ in southern France and $Y_2$ in northern France to behave similarly. 
In both regions, lower temperatures lead to increased heating usage (and thus higher loads), while higher temperatures result in increased cooling usage (also leading to higher loads). 
Therefore, $f^{\star\,1}$ and $f^{\star\,2}$ share the same shape, resulting in the prior $f^{\star\,1} \simeq f^{\star\,2}$. 
This prototypical example of transfer learning is explored in the following paragraphs. Such inexact constraints can be enforced by adding a penalty $\lambda \|C\theta\|_2^2$ in the empirical risk \eqref{eq:risk}, where $\lambda > 0$ is an hyperparameter. (Equivalently, this only consists in replacing $M$ with $(\sqrt{\lambda} C^\top \mid  M^\top)^\top$ in \eqref{eq:risk}.)
This ensures that $\|C\hat \theta\|_2^2$ is small, while allowing the model to learn functions that do not exactly satisfy the constraint. 

These approaches are statistically sound, since under the assumption that $Y_t = f_{\theta^\star}(X_t)+ \varepsilon_t$, where $\theta^\star \in \mathcal{S}$, both estimators have lower errors compared to unconstrained regression. This is true in the sense that, almost surely,
\[\frac{1}{n}\sum_{j=1}^n\| f_{\theta^\star}(X_{t_j}) - f_{\hat \theta_C}(X_{t_j})\|_2^2 + \|M(\theta^\star- \hat \theta_C)\|_2^2\leq \frac{1}{n}\sum_{j=1}^n\| f_{\theta^\star}(X_{t_j}) - f_{\hat \theta}(X_{t_j})\|_2^2 + \|M(\theta^\star- \hat \theta)\|_2^2,\]
where $\hat \theta$ is the unconstrained WeaKL and $\hat \theta_C$ is a WeaKL integrating the constraint $C\theta^\star \simeq 0$ (see Proposition~\ref{prop:prop_lin} and Remark~\ref{rem:comment_prop_lin}).

\paragraph{Transfer learning.} Transfer learning is a framework designed to exploit similarities between different prediction tasks when $d_2 >1$. The simplest case involves predicting multiple targets $Y_1, \hdots, Y_{d_2}$ with similar features $X_1, \hdots, X_{d_2}$.
For example, suppose we want to forecast the electricity demand of $d_2$ cities. Here, $Y_\ell$ is the electricity demand of the city $\ell$, while $X_\ell$ is the average temperature in city $\ell$.
The general function $f^\star$ estimating $(Y_1, \hdots, Y_{d_2})$ can be expressed as $f^\star(X) = f^\star(X_1, \hdots, X_{d_2}) = (f^{\star\,1}(X_1), \hdots, f^{\star\,d_2}(X_{d_2}))$. The transfer learning assumption is $f^{\star\,1} \simeq \cdots \simeq f^{\star\,d_2}$. 
Equivalently, this corresponds to the linear constraint $\theta \in \mathrm{Im}(P)$, where $P = (\mathrm{I}_{2m+1} \mid \cdots \mid \mathrm{I}_{2m+1})^\top$ is a $(2m+1)d_1\times (2m+1)$ matrix. 
Thus, one can apply the framework of the last paragraph on linear constraints as inexact prior using $P$.

\paragraph{Hierarchical forecasting.} Hierarchical forecasting involves predicting multiple time series that are linked by summation constraints. This approach was introduced by \citet{athanasopoulos2009hierarchical} to forecast Australian domestic tourism. Tourism can be analyzed at various geographic scales. 
For example, at time $t$, one could consider the total number $Y_{A,t}$ of tourists in Australia, and the number $Y_{S_i,t}$ of tourists in each of the seven Australian states $S_1,\hdots, S_7$. By definition, $Y_{A,t}$ is the sum of the $Y_{S_i, t}$, which leads to the exact summation constraint $Y_{A,t} = \sum_{i=1}^7 Y_{S_i, t}$. Furthermore, since each state $S_i$ is composed of $z_i$ zones $Z_{i,1}$, $\dots$, $Z_{i, z_i}$, an additional hierarchical level can be introduced. 
Note that the number of zones depends on the state, for a total of 27 zones.
This results in another set of summation constraints
$Y_{S_i, t} =  Y_{Z_{i,1}, t}+\cdots + Y_{Z_{i,z_i}, t}$. 
\begin{figure}
    \centering
    \includegraphics[width=0.7\linewidth, trim={0 3.5cm 0 0},clip]{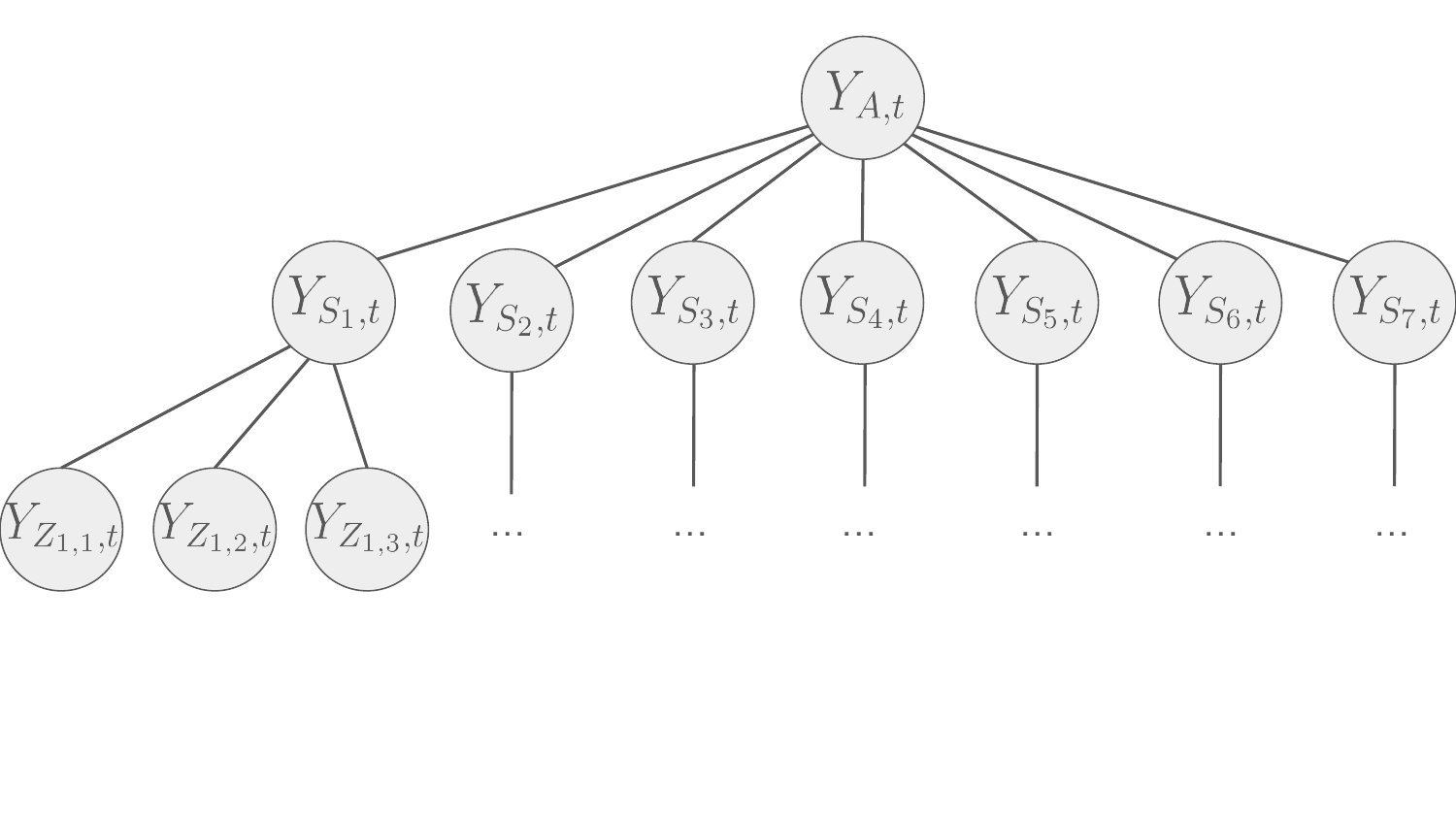}
    \caption{Graph representing the hierarchy of Australian domestic tourism.}
    \label{fig:DAC}
\end{figure}
Overall, the complete set of summation constraints can be represented by a directed acyclic graph, as shown in Figure~\ref{fig:DAC}. 
Alternatively, these constraints can be expressed by a $35 \times 27$ summation matrix $S$ that connects the bottom-level series $Y_b= ( Y_{Z_{1,1}},\hdots, Y_{Z_{7,z_7}})^\top \in \mathbb R^{27}$ to all hierarchical nodes $Y = (Y_{Z_{1,1}},\dots, Y_{Z_{7,z_7}}, Y_{S_1}, \hdots, Y_{S_7}, Y_A)^\top \in \mathbb R^{35}$ through the relation $Y = S Y_b$.
Thus, by letting $\mathbb 1 = (1, \hdots, 1)^\top \in \mathbb{R}^{27}$, and defining $\mathbb 1^{(j)}\in \mathbb R^{27}$ by $\mathbb 1^{(j)}_i = \left\{\begin{array}{cc}
     1 &   \hbox{ if } \sum_{k=1}^{j-1} z_k \leq i \leq \sum_{k=1}^j z_k\\
     0 &   \hbox{ otherwise}
\end{array}\right.$,  we have that $S = (\mathrm{I}_{27} \mid \mathbb 1^{(1)} \mid \cdots \mid \mathbb 1^{(7)} \mid  \mathbb 1)^\top$. 
The goal of hierarchical forecasting is to take advantage of the summation constraints defined by $S$ to improve the predictions of the vector $Y$ representing all hierarchical nodes.

This context can be easily generalized to many time series forecasting tasks. Spatial summation constraints, which divide a geographic space into different subspaces, have been applied in areas such as electricity demand forecasting \citep{bregere2022online}, electric vehicle charging demand forecasting \citep{amara-ouali2024forecasting}, and tourism forecasting \citep{Wickramasuriya2019optimal}.
Summation constraints also arise in multi-horizon forecasting, where, for example, an annual forecast must equal the sum of the corresponding monthly forecasts \citep{kourentzes2019cross}.
Finally, they also appear when goods are categorized into different groups \citep{pennings2017integrated}.

There are two main approaches to hierarchical forecasting. The first, known as forecast reconciliation, attempts to improve an existing estimator $\hat{Y}$ of the hierarchical nodes $Y$ by multiplying $\hat{Y}$ by a so-called reconciliation matrix $P$, so that the new estimator $P \hat Y$ satisfies the summation constraints. 
Formally, it is required that $\mathrm{Im}(P) \subseteq \mathrm{Im}(S)$, where $S$ is the summation matrix. 
The goal is for $P\hat{Y}$ to have less error than $\hat{Y}$. 
The strengths of this approach are its low computational cost and its ability to seamlessly integrate with pre-existing forecasts. 
Various reconciliation matrices, such as the orthogonal projection $P = S(S^\top S)^{-1}S$ on $\mathrm{Im}(S)$ (see the paragraph above on linear constraints), have been shown to reduce forecasting errors and to even be optimal under certain assumptions 
\citep{Wickramasuriya2019optimal}. Another complementary approach is to incorporate the hierarchical structure of the problem directly into the training of the initial estimator $\hat{Y}$ 
\citep{rangapuram21end}.
While this method is more computationally intensive, it provides a more comprehensive solution than reconciliation methods because it uses the hierarchy not only to shape the regression function, but also to inform the learning of its parameters. In this paper, we build on this approach to design three new estimators, all of which are implemented in Section~\ref{sec:tourism}.

As for now, we denote by $\ell_1$ the total number of nodes and $\ell_2 \leq \ell_1$ the number of bottom nodes. Thus, $Y=(Y_{\ell})_{1\leq \ell \leq \ell_1}^\top$ represents the global vector of all nodes, while $Y_b=(Y_{\ell})_{1\leq \ell \leq \ell_2}^\top$ represents the vector of the bottom nodes.
The $\ell_1 \times \ell_2$ summation matrix $S$ is defined so that, for all time index~$t$, the summation identity $Y_t = S Y_{b,t}$ is satisfied. 

\paragraph{Estimator 1. Bottom-up approach: WeaKL-BU.} In the bottom-up approach, models are fitted only for the bottom-level series $Y_b$, resulting in a vector of estimators $\hat{Y}_b$. The remaining levels are then estimated by $\hat{Y} = S \hat{Y}_b$, where $S$ is the summation matrix. 

To achieve this, forecasts for each bottom node $1 \leq \ell \leq \ell_2$ are constructed using a set of explanatory variables $X_\ell \in \mathbb R^{d_\ell}$ specific to that node. Together, these explanatory variables $X_1, \hdots, X_{\ell_2}$ form the feature $X\in \mathbb R^{d_1+\dots +d_{\ell_2}}$. A straightforward choice of features are the lags of the target variable, i.e., $X_{\ell, t} = Y_{\ell, t-1}$, though many other choices are possible. Next, for each bottom node $1 \leq \ell \leq \ell_2$, we fit a parametric model $f_{\theta_\ell}(X_{\ell, t})$ to predict the series $Y_{\ell, t}$. 
Each function $f_{\theta_\ell}$ is parameterized by a mapping $\phi_\ell$ (e.g., a Fourier map or an additive model) and a coefficient vector $\theta_\ell$, such that
$f_{\theta_\ell}(X_{\ell,t}) = \langle \phi_\ell(X_{\ell,t}), \theta_\ell \rangle$.
Therefore, the model for the lower nodes $Y_{b, t}$ can be expressed as $\mathbb \Phi_t \theta$, where  $\theta = (\theta_1, \hdots, \theta_{\ell_2})^\top$ is the vector of all coefficients, and $\mathbb{\Phi}_t$ is the feature matrix at time $t$ defined in \eqref{eq:feature_matrix}. 
Overall, the model for all levels $Y_t = S Y_{b, t}$ is $S \mathbb{\Phi}_t \theta$, and the empirical risk corresponding to this problem is given by
\[L(\theta) = \frac{1}{n}\sum_{j=1}^n\|\Lambda( S\mathbb \Phi_{t_j} \theta - Y_{t_j})\|_2^2 + \|M\theta\|_2^2,\]
where $\Lambda$ is a $\ell_1\times \ell_1$ diagonal matrix with positive coefficients,  and $M$ is a penalty matrix that depends on the $\phi_\ell$ mappings, as in Section~\ref{sec:shape}. 

Since $\Lambda$ scales the relative importance of each node in the learning process, the choice of its coefficients plays a critical role in the performance of the estimator. 
In the experimental Section~\ref{sec:tourism}, $\Lambda$ will be learned through hyperparameter tuning. Typically, $\Lambda_{\ell, \ell}$ should be large when $\mathrm{Var}(Y_\ell | X_\ell)$ is low---that is, the more reliable $Y_\ell$ is as a target \citep{Wickramasuriya2019optimal}. From~\eqref{eq:feature_matrix}, we deduce that the minimizer $\hat \theta$ of the empirical risk is 
    \begin{equation}
    \label{eq:pikl-bu}
        \hat \theta = \Big(\Big(\sum_{j=1}^n \mathbb \Phi_{t_j}^\ast S^\ast\Lambda^\ast \Lambda S\mathbb \Phi_{t_j}\Big) + n M^\ast M\Big)^{-1} \sum_{j=1}^n \mathbb \Phi_{t_j}^\ast \Lambda^\ast \Lambda Y_{t_j}.
    \end{equation}
We call $\hat{\theta}$ the {\it WeaKL-BU}. Setting $\Lambda = \mathrm{I}_{\ell_1}$, i.e., the identity matrix, results in treating all hierarchical levels equally, which is the setup of \citet{rangapuram21end}. On the other hand, setting $\Lambda_{\ell, \ell} = 0$ for all $\ell \geq \ell_2$ leads to learning each bottom node independently, without using any information from the hierarchy. 
This is the traditional bottom-up approach.

\paragraph{Estimator 2. Global hierarchy-informed approach: WeaKL-G.} 
The context is similar to the bottom-up approach, but here models are fitted for all nodes $1 \leq \ell \leq \ell_1$, using local explanatory variables $X_\ell \in \mathbb{R}^{d_\ell}$, where $d_\ell \geq 1$. Thus, the model for $Y_{t}$ is given by $\mathbb \Phi_t \theta$, where $\theta = (\theta_1, \hdots, \theta_{\ell_1})^\top$ is the vector of coefficients and $\mathbb{\Phi}_t$ is the feature matrix at time $t$ defined in \eqref{eq:feature_matrix}.
To ensure that the hierarchy is respected, we introduce a penalty term:
\[
\|\Gamma(S\Pi_b\mathbb \Phi_t\theta-\mathbb \Phi_t\theta)\|_2^2  = \|\Gamma(S\Pi_b-\mathrm{I}_{\ell_1})\mathbb \Phi_t\theta\|_2^2,
\]
where $\Gamma$ is a positive diagonal matrix and $\Pi_b$ is the projection operator on the bottom level, defined as $\Pi_b \theta = (\theta_1, \hdots, \theta_{\ell_2})^\top$. As in the bottom-up case, $\Gamma$ encodes the level of trust assigned to each node. 
In Section~\ref{sec:tourism}, we learn $\Gamma$ through hyperparameter tuning. This results in the empirical risk
\[L(\theta) = \frac{1}{n}\sum_{j=1}^n\|\mathbb \Phi_{t_j} \theta - Y_{t_j}\|_2^2 + \frac{1}{n}\sum_{j=1}^n\|\Gamma(S\Pi_b-\mathrm{I}_{\ell_1})\mathbb \Phi_{t_j}\theta\|_2^2 + \|M\theta\|_2^2.\]
where $M$ is a penalty matrix that depends on the $\phi_\ell$ mappings, as in Section~\ref{sec:shape}.
This empirical risk is similar to the one proposed by \citet{Zheng2023coherent}, where a penalty term is used to enforce hierarchical coherence during the learning process.
From \eqref{eq:feature_matrix}, we deduce that the  minimizer is given by 
    \begin{equation}
    \label{eq:pikl-G}
        \hat \theta = \Big(\sum_{j=1}^n (\mathbb \Phi_{t_j}^\ast\mathbb \Phi_{t_j}+ \mathbb \Phi_{t_j}^\ast(\Pi_b^\ast S^\ast-\mathrm{I}_{\ell_1})\Gamma^\ast \Gamma(S\Pi_b-\mathrm{I}_{\ell_1})\mathbb \Phi_{t_j})+nM^\ast M\Big)^{-1} \sum_{j=1}^n \mathbb \Phi_{t_j}^\ast Y_t.
    \end{equation}
We refer to $\hat{\theta}$  as the {\it WeaKL-G}.
The fundamental difference between \eqref{eq:pikl-bu} and \eqref{eq:pikl-G} is that the WeaKL-BU estimator only learns parameters for the $\ell_2$ bottom nodes, whereas the WeaKL-G estimators learns parameters for all nodes. We emphasize that WeaKL-BU and WeaKL-G follow different approaches. While WeaKL-BU adjusts the lower-level nodes and then uses the summation matrix $S$ to estimate the higher levels, WeaKL-G relies directly on global information, which is subsequently penalized by $S$. In the next paragraph, we complement the WeaKL-BU estimator by adding transfer learning constraints.

\paragraph{Estimator 3. Hierarchy-informed transfer learning: WeaKL-T.} In many hierarchical forecasting applications, the targets $Y_{\ell}$ are of the same nature throughout the hierarchy. Consequently, we often expect them to be explained by similar explanatory variables $X_{\ell}$ and to have similar regression functions estimators $f_{\hat{\theta}_\ell}$ \citep[e.g.,][]{leprince2023hierarchical}. For this reason, we propose an algorithm that combines WeaKL-BU with transfer learning.

Therefore, we assume that there is a subset $J \subseteq \{1,
\hdots, \ell_2\}$ of similar nodes and weights $(\alpha_i)_{i\in J}$ such that we expect $\alpha_i f_{\hat \theta_{i}}(X_{i,t}) \simeq \alpha_j f_{\hat \theta_{j}}(X_{j,t})$ for $i, j \in J$. 
In particular, there is an integer $D$ such that $\theta_j \in \mathbb{C}^{D}$ for all $j\in J$.
Therefore, denoting by $\Pi_J$ the projection on $J$ such that $\Pi_J\theta = (\theta_j)_{j\in J}\in \mathbb C^{D|J|}$, this translates into the constraint that $\Pi_J \theta \in \mathrm{Im}(M_J)$ where $M_J = (\alpha_1 \mathrm{I}_{D}, \hdots, \alpha_{|J|} \mathrm{I}_{D})^\top$.
As explained in the paragraph on linear constraints, we enforce this inexact constraint by penalizing the empirical risk with the addition of the term $\|(\mathrm{I}_{D|J|}-P_J)\Pi_J\theta\|_2^2$, where $P_J = M_J(M_J^\ast M_J)^{-1}M_J^\ast$ is the orthogonal projection onto the image of $M_J $.
This leads to the empirical risk
\[L(\theta) = \frac{1}{n}\sum_{j=1}^n\|\Lambda( S\mathbb \Phi_{t_j} \theta - Y_{t_j})\|_2^2+\lambda \|(\mathrm{I}_{D|J|}-P_{J})\Pi_{J}\theta\|_2^2 + \|M\theta\|_2^2,\]
where $M$ is a penalty matrix that depends on the $\phi_\ell$ mappings, as in Section~\ref{sec:shape}.
We call {\it WeaKL-T} the minimizer $\hat \theta$ of $L$. It is given by
    \begin{equation}
    \label{eq:pikl-T}
        \hat \theta = \Big(\Big(\sum_{j=1}^n \mathbb \Phi_{t_j}^\ast S^\ast\Lambda^\ast \Lambda S\mathbb \Phi_{t_j}\Big)+  n\lambda \Pi_{J}^\ast(\mathrm{I}_{D|J|}-P_{J})\Pi_{J}+nM^\ast M\Big)^{-1} \sum_{j=1}^n \mathbb \Phi_{t_j}^\ast\Lambda^\ast \Lambda Y_{t_j}.
    \end{equation}

\subsection{Application to tourism forecasting}
\label{sec:tourism}
\paragraph{Hierarchical forecasting and tourism.}

In this experiment, we aim to forecast Australian domestic tourism using the dataset from \citet{Wickramasuriya2019optimal}. The dataset includes monthly measures of Australian domestic tourism from January 1998 to December 2016, resulting in $n = 216$ data points. 
Each month, domestic tourism is measured at four spatial levels and one categorical level, forming a five-level hierarchy. At the top level, tourism is measured for Australia as a whole. 
It is then broken down spatially into $7$ states, $27$ zones, and $76$ regions. 
Then, for each of the $76$ regions, four categories of tourism are distinguished  according to the purpose of travel: holiday, visiting friends and
relatives (VFR), business, and other. This gives a total of five levels (Australia, states, zones, regions, and categories), with $\ell_2 = 76 \times 4 = 304$ bottom nodes, and $\ell_1 = 1 + 7 + 27 + 76 + \ell_2 = 415$ total nodes.

\paragraph{Benchmark.} The goal is to forecast Australian domestic tourism one month in advance. Models are trained on the first $80\%$ of the dataset and evaluated on the last $20\%$. 
Similar to \citet{Wickramasuriya2019optimal}, we only consider autoregressive models with lags from one month to two years.
This setting is particularly interesting because, although each time series can be reasonably fitted using the $216$ data points, the total number of targets $\ell_1$ exceeds $n$. Consequently, the higher levels cannot be naively learned from the lags of the bottom level time series through linear regression.

The bottom-up (BU) model involves running $304$ linear regressions $\hat Y^{\mathrm{BU}}_{\ell,t} = \sum_{j=1}^{24}a_{\ell, j}Y_{\ell,t-j}$ for $1\leq \ell \leq \ell_2$, where $Y_{\ell,t-j}$ is the lag of $Y_{\ell, t}$ by $j$ months. 
The final forecast is then computed as $\hat Y^{\mathrm{BU}}_{t} = S\hat Y^{\mathrm{BU}}_{\ell,t}$,  where $S$ is the summation matrix.  
The Independent (Indep) model involves running separate linear regressions for each target time series using its own lags. This results in $415$ linear regressions of the form $\hat Y^{\mathrm{Indep}}_{\ell,t} = \sum_{j=1}^{24}a_{\ell, j}Y_{\ell,t-j}$ for $1\leq \ell \leq \ell_1$. 
Rec-OLS is the estimator resulting from OLS adjustment of the Indep estimator, i.e., taking $P = S(S^\ast S)^{-1}S$ \citep{Wickramasuriya2019optimal}. 
MinT refers to the estimator derived from the minimum trace adjustment of the Indep estimator \citep[see MinT(shrinkage) in][]{Wickramasuriya2019optimal}.
PIKL-BU refers to the estimator \eqref{eq:pikl-bu}, where, for all $1\leq \ell \leq 304$, $X_{\ell,t} = (Y_{\ell, t-j})_{1\leq j \leq 24}$ and $\phi_\ell(x) = x$.
PIKL-G is the estimator \eqref{eq:pikl-G}, where, for all $1\leq \ell \leq 415$, $X_{\ell,t} = (Y_{\ell, t-j})_{1\leq j \leq 24}$ and $\phi_\ell(x) = x$. 
Finally, PIKL-T is the estimator \eqref{eq:pikl-T}, where $X_{\ell,t} = (Y_{\ell, t-j})_{1\leq j \leq 24}$ and $\phi_\ell(x) = x$. 
In the latter model, all the auto-regressive effects are penalized to enforce uniform weights, which means that $\alpha_\ell = 1$ and $J = \{1, \hdots, \ell_2\}$ in \eqref{eq:pikl-T}.
The hyperparameter tuning process to learn the matrix $\Lambda$ for the WeaKLs is detailed in Appendix~\ref{sec:hierarchical_details}.

\paragraph{Results.} Table~\ref{table_australia} shows the results of the experiment. The mean square errors (MSE) are computed for each hierarchical level and aggregated under {\it All levels}. 
Their standard deviations are estimated using block bootstrap with blocks of length $12$.
The models are categorized based on the features they utilize.
We observe that the WeaKL-type estimators consistently outperform all other competitors in every case. This highlights the advantage of incorporating constraints to enforce the hierarchical structure of the problem, leading to an improved learning process.

\begin{table}[H]
\centering
\caption{Benchmark in forecasting Australian domestic tourism} 
\begin{tabular}{lcccccc}
\toprule
 &  & &$\;$\hfill  MSE  & ($\times 10^6$) \hfill $\;$&   &  \\
 \cmidrule{2-7}%
 & Australia & States & Zones & Regions & Categories & All levels \\
\midrule
\textit{Bottom data} &&&&&&\\
BU & $5.3\!\pm\!0.5$ & $2.0\!\pm\!0.2$ & $1.37\!\pm\!0.05$ & $\mathbf{1.19\!\pm\!0.02}$ & $\mathbf{1.17\!\pm\!0.03}$ & $11.0\!\pm\!0.7$ \\
WeaKL-BU & $\mathbf{4.5\!\pm\!0.5}$ & $\mathbf{1.9\!\pm\!0.3}$ & $\mathbf{1.34\!\pm\!0.05}$ & $\mathbf{1.19\!\pm\!0.03}$ & $\mathbf{1.17\!\pm\!0.03}$ & $\mathbf{10.1\!\pm\!0.6}$ \\
\midrule
\textit{Own lags} &&&&&&\\
Indep & $3.6\!\pm\!0.6$ & $1.8\!\pm\!0.2$ & $1.42\!\pm\!0.05$ & $1.23\!\pm\!0.03$ & $1.17\!\pm\!0.03$ & $9.2\!\pm\!0.7$ \\
WeaKL-G & $\mathbf{3.6\!\pm\!0.5}$ & $\mathbf{1.8\!\pm\!0.2}$ & $\mathbf{1.37\!\pm\!0.05}$ & $\mathbf{1.18\!\pm\!0.03}$ & $\mathbf{1.15\!\pm\!0.03}$ & $\mathbf{9.0\!\pm\!0.7}$ \\
\midrule
\textit{All data} &&&&&&\\
Rec-OLS & $3.5\!\pm\!0.5$ & $1.8\!\pm\!0.2$ & $1.35\!\pm\!0.05$ & $1.18\!\pm\!0.02$ & $1.17\!\pm\!0.03$ & $8.9\!\pm\!0.7$ \\
MinT & $3.6\!\pm\!0.4$ & $1.7\!\pm\!0.1$ & $1.29\!\pm\!0.05$ & $\mathbf{1.15\!\pm\!0.03}$ & $1.17\!\pm\!0.03$ & $8.9\!\pm\!0.5$ \\
WeaKL-T & $\mathbf{3.1\!\pm\!0.3}$ & $\mathbf{1.7\!\pm\!0.1}$ & $\mathbf{1.27\!\pm\!0.05}$ & $\mathbf{1.15\!\pm\!0.02}$ & $\mathbf{1.12\!\pm\!0.03}$ & $\mathbf{8.3\!\pm\!0.4}$ \\
\bottomrule
\end{tabular}
\label{table_australia}
\end{table}

\section{Conclusion}

In this paper, we have shown how to design empirical risk functions that integrate common linear constraints in time series forecasting. For modeling purposes, we distinguish between shape constraints (such as additive models, online adaptation after a break, and forecast combinations) and learning constraints (including transfer learning, hierarchical forecasting, and differential constraints). 
These empirical risks can be efficiently minimized on a GPU, leading to the development of an optimized algorithm, which we call WeaKL. 
We have applied WeaKL to three real-world use cases---two in electricity demand forecasting and one in tourism forecasting---where it consistently outperforms current state-of-the-art methods, demonstrating its effectiveness in structured forecasting problems.

Future research could explore the integration of additional constraints into the WeaKL framework. For example, the current approach does not allow for forcing the regression function $f_\theta$ to be non-decreasing or convex. However, since any risk function $L$ of the form \eqref{eq:risk} is convex in $\theta$, the problem can be formulated as a linearly constrained quadratic program. While this generally increases the complexity of the optimization, it can also lead to efficient algorithms for certain constraints. In particular, when $d=1$, imposing a non-decreasing constraint on $f_\theta$ reduces the problem to isotonic regression, which has a computational complexity of $O(n)$ \citep{wright1980isotonic}.

%% file: input-supplementary.tex
\section{Proofs}
The purpose of this appendix is to provide detailed proofs of the theoretical results presented in the main article. Appendix~\ref{proof:kernel} elaborates on the formula that characterizes the unique minimizer of the WeaKL empirical risks, while Appendix~\ref{sec:constraints} discusses the integration of linear constraints into the empirical risk framework.

\subsection{A useful lemma}
\begin{lemma}[Full rank]
    The matrix \[\tilde M = \frac{1}{n}\Big( \sum_{j=1}^n \mathbb \Phi_{t_j}^\ast \Lambda^\ast \Lambda\mathbb \Phi_{t_j}\Big) +  M^\ast M\] is invertible. Moreover, for all $\theta\in\mathbb C^{\dim \theta}$, $\theta^\star \tilde M \theta \geq \lambda_{\min}(\tilde M)\|\theta\|_2^2$, where $\lambda_{\min}(\tilde M)$ is the minimum eigenvalue of $\tilde M$.
    \label{lemma:full}
\end{lemma}
\begin{proof}
    First, we note that $\tilde M$ is a positive Hermitian square matrix. Hence, the spectral theorem guarantees that $\tilde M$ is diagonalizable in an orthogonal basis of $\mathbb C^{\dim(\theta)}$ with real eigenvalues. In particular, it admits a positive square root, and the min-max theorem states that $\theta^\ast \tilde M \theta = \|\tilde M^{1/2} \theta\|_2^2 \geq \lambda_{\min}(\tilde M^{1/2})^2\|\theta\|_2^2 = \lambda_{\min}(\tilde M)\|\theta\|_2^2$. This shows the second statement of the lemma.
    
    Next, for all $\theta \in \mathbb C^{\dim \theta}$, $\theta^\ast \tilde M \theta \geq \theta^\ast  M^\ast M \theta$.
    Since $M$ is full rank,  $\mathrm{rank}(M) = \dim(\theta)$. Therefore, $\tilde M \theta = 0 \Rightarrow \theta^\ast \tilde M \theta = 0 \Rightarrow \theta^\ast M^\ast M \theta = 0  \Rightarrow \|M\theta\|_2^2 = 0 \Rightarrow M\theta = 0 \Rightarrow \theta = 0$. Thus, $\tilde M$ is injective and, in turn, invertible.
\end{proof}

\subsection{Proof of Proposition~\ref{prop:emp_risk_min}}
\label{proof:kernel}
The function $L: \mathbb C^{\dim(\theta)} \to \mathbb R^+$ can be written as
\[L(\theta) = \frac{1}{n}\Big( \sum_{j=1}^n (\mathbb \Phi_{t_j}\theta- Y_{t_j})^\ast \Lambda^\ast \Lambda(\mathbb \Phi_{t_j}\theta- Y_{t_j})\Big)  + \theta^\ast M^\ast M\theta.\]
Recall that the matrices $\Lambda$ and $M$ are assumed to be injective. 
Observe that $L$ can be expanded as
\[L(\theta + \delta \theta) = L(\theta) + 2 \mathrm{Re}(\langle \tilde M\theta-\tilde Y, \delta \theta\rangle) + o(\|\delta \theta\|_2^2),\]
where $\tilde Y = \frac{1}{n} \sum_{j=1}^n \mathbb \Phi_{t_j}^\ast \Lambda^\ast \Lambda Y_{t_j}.$
This shows that $L$ is differentiable and that its differential at $\theta$ is the function $dL_\theta: \delta \theta \mapsto 2 \mathrm{Re}(\langle \tilde M\theta - \tilde Y, \delta \theta\rangle)$.
Thus, the critical points $\theta$ such that $dL_{\theta} = 0$ satisfy 
\[\forall\; \delta \theta \in \mathbb C^{\dim(\theta)}, \; \mathrm{Re}(\langle \tilde M\theta- \tilde Y, \delta \theta\rangle) = 0.\]
Taking $\delta \theta = \tilde M \theta- \tilde Y$ shows that $\|\tilde M\theta- \tilde Y\|_2^2 = 0$, i.e., $\tilde M\theta = \tilde Y$. From Lemma~\ref{lemma:full}, we deduce that $\theta = \tilde M^{-1} \tilde Y$, which is exactly the $\hat \theta_n$ in \eqref{eq:weakl}.

From Lemma~\ref{lemma:full}, we also deduce that, for all $\theta$ such that $\|\theta\|_2$ is large enough, one has $L(\theta) \geq \lambda_{\min}(\tilde M)\|\theta\|_2^2/2$. 
Since $L$ is continuous, it has at least one global minimum. Since the unique critical point of $L$ is $\hat \theta_n$, we conclude that $\hat \theta_n$ is the unique minimizer of $L$. 
\subsection{Orthogonal projection and linear constraints}
\label{sec:constraints}
\begin{lemma}[Orthogonal projection]
    \label{lem:ortho}
    Let $\ell_1, \ell_2 \in \mathbb N^\star$. Let $P$ be an injective $\ell_1 \times \ell_2$ matrix with coefficients in $\mathbb C$. Then
    $C = \mathrm{I}_{\ell_1} - P(P^\ast P)^{-1}P^\ast$ is the orthogonal projection on $\mathrm{Im}(P)^\perp$, where $\mathrm{Im}(P)$ is the image  of $P$ and $\mathrm{I}_{\ell_1}$ is the $\ell_1\times \ell_1$ identity matrix.
\end{lemma}
\begin{proof}
First, we show that $P^\ast P$ is an $\ell_2 \times \ell_2$ matrix of full rank. Indeed, for all $x\in \mathbb C^{\ell_2}$, one has $P^\ast P x = 0 \Rightarrow x^\ast P^\ast P x = 0 \Rightarrow \|Px\|_2^2 = 0$. Since $P$ is injective, we deduce that $\|Px\|_2^2 = 0 \Rightarrow x = 0$. This means that $\ker P^\ast P = \{0\}$, and so that $P^\ast P$ is full rank. Therefore, $(P^\ast P)^{-1}$ is well defined.

Next, let $C_1 = P(P^\ast P)^{-1}P^\ast$. Clearly, $C_1^2 = C_1$, i.e., $C_1$ is a projector. Since $C_1^\ast = C_1$, we deduce that $C_1$ is an orthogonal projector.
 In addition, since $C_1 = P\times ((P^\ast P)^{-1}P^\ast)$, $\mathrm{Im}(C_1) \subseteq \mathrm{Im}(P)$. Moreover, if $x \in \mathrm{Im}(P)$, there exists a vector $ z$ such that $x = Pz$, and $C_1x = P(P^\ast P)^{-1}P^\ast Pz = Pz =x$. Thus, $x \in \mathrm{Im}(C_1)$. This shows that $\mathrm{Im}(C_1) = \mathrm{Im}(P)$. We conclude that $C_1$ is the orthogonal projection on $\mathrm{Im}(P)$ and, in turn, that $C = \mathrm{I}_{\ell_1} - C_1$ is the orthogonal projection on $\mathrm{Im}(P)^\perp$.
\end{proof}

The following proposition shows that, given the exact prior knowledge $C\theta^\star = 0$, enforcing the linear constraint  $C\theta = 0$ almost surely improves the performance of WeaKL.
\begin{proposition}[Constrained estimators perform better.]
\label{prop:prop_lin}
    Assume that $Y_t = f_{\theta^\star}(X_t) + \varepsilon_t$ and that $\theta^\star$ satisfies the constraint $C\theta^\star = 0$, for some matrix $C$. (Note that we make no assumptions about the distribution of the time series $(X, \varepsilon)$.)
    Let $\Lambda$ and $M$ be injective matrices, and let $\lambda \geq 0$ be a hyperparameter. 
    Let $\hat \theta$ be the WeaKL given by \eqref{eq:weakl} and let $\hat \theta_C$ be the WeaKL obtained by replacing $M$ with $(\sqrt{\lambda}C^\top \mid M^\top)^\top$ in \eqref{eq:weakl}.
    Then, almost surely,
    \[\frac{1}{n}\sum_{j=1}^n\| f_{\theta^\star}(X_{t_j}) - f_{\hat \theta_C}(X_{t_j})\|_2^2 + \|M(\theta^\star- \hat \theta_C)\|_2^2\leq \frac{1}{n}\sum_{j=1}^n\| f_{\theta^\star}(X_{t_j}) - f_{\hat \theta}(X_{t_j})\|_2^2 + \|M(\theta^\star- \hat \theta)\|_2^2.\]
\end{proposition}
\begin{proof} Recall from \eqref{eq:weakl} that
\[
    \hat \theta = P^{-1} \sum_{j=1}^n \mathbb \Phi_{t_j}^\ast \Lambda^\ast \Lambda Y_{t_j} \quad \mbox{and} \quad 
     \hat \theta_C = \big(P+ \lambda n C^\ast C\big)^{-1} \sum_{j=1}^n \mathbb \Phi_{t_j}^\ast \Lambda^\ast \Lambda Y_{t_j},
\]
where $P = ( \sum_{j=1}^n \mathbb \Phi_{t_j}^\ast \Lambda^\ast \Lambda\mathbb \Phi_{t_j}) + n M^\ast M$.
Since $C\theta^\star = 0$, we see that 
\begin{equation}
    \theta^\star = \big(P+ \lambda n C^\ast C\big)^{-1}P\theta^\star
    \label{eq:theta_star}.
\end{equation}
Subtracting \eqref{eq:theta_star} to, respectively, $\hat \theta$ and $\hat \theta_C$, we obtain
\[
     \theta^\star- \hat \theta = P^{-1/2} \Delta \quad \mbox{and}\quad 
    \theta^\star - \hat \theta_C = \big( P +\lambda n C^\ast C\big)^{-1} P^{1/2} \Delta,
\]
where \[\Delta = P^{-1/2}\Big(P\theta^\star- \sum_{j=1}^n \mathbb \Phi_{t_j}^\ast \Lambda^\ast \Lambda Y_{t_j}\Big).\] 
Moreover, according to the Loewner order \citep[see, e.g.,][Chapter~7.7]{Horn2012matrix}, we have that
$P^{-1/2}C^\ast C P^{-1/2} \geq 0$ and $(P^{-1/2}C^\ast C P^{-1/2})^2 \geq 0$. 
(Indeed, since $P$ is Hermitian, so is $P^{-1/2}C^\ast C P^{-1/2}$.)
Therefore, $(\mathrm{I} +\lambda n  P^{-1/2}C^\ast C P^{-1/2})^2 \geq \mathrm{I}$ and $( \mathrm{I} +\lambda n  P^{-1/2}C^\ast C P^{-1/2})^{-2} \leq \mathrm{I}$ \citep[see, e.g.,][Corollary~7.7.4]{Horn2012matrix}.
Consequently,
\[\|P^{1/2}(\theta^\star - \hat \theta_C)\|_2^2 = \Delta^\ast \big( \mathrm{I} +\lambda n  P^{-1/2}C^\ast C P^{-1/2}\big)^{-2} \Delta \leq \|\Delta\|_2^2 = \|P^{1/2}(\theta^\star - \hat \theta)\|_2^2.\]
Observing that $\|P^{1/2}(  \theta^\star-\hat \theta_C)\|_2^2 = \frac{1}{n}\sum_{j=1}^n\| f_{\theta^\star}(X_{t_j}) - f_{\hat \theta_C}(X_{t_j})\|_2^2 + \|M(\theta^\star- \hat \theta_C)\|_2^2$ and $\|P^{1/2}( \theta^\star- \hat \theta)\|_2^2 = \frac{1}{n}\sum_{j=1}^n\| f_{\theta^\star}(X_{t_j}) - f_{\hat \theta}(X_{t_j})\|_2^2 + \|M(\theta^\star- \hat \theta)\|_2^2$ concludes the proof.
\end{proof}
\begin{remark}
\label{rem:comment_prop_lin}
    Taking the limit $\lambda \to \infty$ in Proposition~\ref{prop:prop_lin} does not affect the result and corresponds to restricting the parameter space to $\ker(C)$, meaning that, in this case, $C \hat \theta_C = 0$.
    
Note also that the proposition is extremely general, as it holds almost surely without requiring any assumptions on either $X$ or $\varepsilon$.
Here, the error of $\hat \theta$ is measured by 
\[\frac{1}{n}\sum_{j=1}^n\| f_{\theta^\star}(X_{t_j}) - f_{ \hat \theta}(X_{t_j})\|_2^2 + \|M(\theta^\star- \hat \theta)\|_2^2,\]
which quantifies both the error of $\hat \theta$ at the points $X_{t_j}$ and in the $M$ norm.
Under additional assumptions on $X$ and $\varepsilon$, this discretized risk can be shown to converge to the $L^2$ error, $\mathbb E\| f_{\theta^\star}(X) - f_{ \hat \theta}(X)\|_2^2$, using Dudley’s theorem \citep[see, e.g., Theorem~5.2 in the Supplementary Material of][]{doumeche2023convergence}.

However, the rate of this convergence of $\hat \theta$ to $\theta^\star$ depends on the properties of $C$ and $M$, as well as the growth of $\dim(\theta)$ with $n$.
For instance, when the penalty matrix $M$ encodes a PDE prior, the analysis becomes particularly challenging and remains an open question in physics-informed machine learning.
Therefore, we leave the study of this convergence outside the scope of this article.
\end{remark}

\section{More WeaKL models}
\subsection{Forecast combinations}
\label{sec:combination}
To forecast a time series $Y$, different models can be used, each using different implementations and sets of explanatory variables. Let $p$ be the number of models and let $\hat{Y}^1_t, \ldots, \hat{Y}^p_t$ be the respective estimators of $Y_t$.
The goal is to determine the optimal weighting of these forecasts, based on their performance evaluated over the time points $t_1 \leq  \cdots \leq t_n$. 
Therefore, in this setting, $X_t = (t, \hat{Y}^1_t, \ldots, \hat{Y}^p_t)$, and the goal is to find the optimal function linking $X_t$ to $Y_t$.
Note that, to avoid overfitting, we assume that the forecasts $\hat{Y}^1_t, \ldots, \hat{Y}^p_t$ were trained on time steps before~$t_1$.
This approach is sometimes referred to as the online aggregation of experts \citep{Remlinger2023expert, Antoniadis2024Aggregation}. Such forecast combinations are widely recognized to significantly improve the performance of the final forecast \citep{timmermann2006handbook, vilmarest2022state,petropoulos2022forecasting, amara-ouali2024forecasting}, as they leverage the strengths of the individual predictors. 

Formally, this results in the model
\[f_\theta(X_t) = \sum_{\ell=1}^p (p^{-1}+ h_{\theta_\ell}(t) )\hat Y^\ell_{t},\]
where $h_{\theta_\ell}(t) = \langle \phi(t), \theta_\ell\rangle$, $\phi$ is the Fourier map $\phi(t) =(\exp(i k t/2))_{-m\leq k \leq m}^\top$, and $\theta_\ell \in \mathbb{C}^{2m+1}$.
The $p^{-1}$ term introduces a bias, ensuring that $h_{\theta_\ell} = 0$ corresponds to a uniform weighting of the forecasts $\hat Y^\ell$.
The function $f^\star$ is thus estimated by minimizing the loss
    \[L(\theta) = \frac{1}{n}\sum_{j=1}^n \Big|\Big(\sum_{\ell=1}^p (p^{-1}+ h_{\theta_\ell}(t_j) )\hat Y^\ell_{t_j}\Big) - Y_{t_j}\Big|^2  + \sum_{\ell=1}^{p} \lambda_\ell \|h_{\theta_\ell}\|_{H^s}^2,\]
    where $\lambda_\ell > 0$ are hyperparameters.
Again, a common choice for the smoothing parameter is to set $s = 2$. 
Let $\phi_1(X_t) = 
    (
    (\hat Y^\ell_{t}\exp(ik t/2))_{- m\leq k \leq  m})_{\ell=1}^p)^\top \in \mathbb C^{(2m+1)p}$.
The Fourier coefficients that minimize the empirical risk are given by
\[
\hat \theta  = ({\mathbb{\Phi}} ^\ast {\mathbb{\Phi}} + n M^\ast  M)^{-1}{\mathbb \Phi}^\ast   \mathbb W,
\]
where $\mathbb W = (W_{t_1}, \hdots, W_{t_n})^\top$ is such that $W_t = Y_t - p^{-1}\sum_{\ell=1}^p \hat Y^\ell_t$,
\[M = \begin{pmatrix}
        \sqrt{\lambda_1} D& 0  & 0\\
        0 & \ddots&  0\\
        0 & 0& \sqrt{\lambda_{d_1}} D
    \end{pmatrix},\]
and $D$ is the $(2m+1)\times (2m+1)$ diagonal matrix
$D =\mathrm{Diag}((\sqrt{1+k^{2s}})_{-m\leq k\leq m})$. 

\subsection{Differential constraints}
\label{sec:diff}
As discussed in the introduction, some time series obey physical laws and can be expressed as solutions of PDEs. Physics-informed kernel learning (PIKL) is a kernel-based method developed by \citet{doumèche2024physicsinformedkernellearning} to incorporate such PDEs as constraints. It can be regarded as a specific instance of the WeaKL framework proposed in this paper. In effect, given a bounded Lipschitz domain $\Omega$ and a linear differential operator $\mathscr D$, using the model $f_{ \theta}(x) = \langle \phi(x), \theta\rangle$, where $\phi(x) = (\exp(i  \langle x, k \rangle / 2) )_{\|k\|_\infty \leq m}$ is the Fourier map and $\theta$ represents the Fourier coefficients, the PIKL approach shows how to construct a matrix $M$ such that
\[
\int_\Omega \mathscr{D}(f_\theta, u)^2 \, du = \|M \theta\|_2^2.
\]
Thus, to incorporate the physical prior $\forall x \in \Omega,\; \mathscr D(f^\star, x) = 0$ into the learning process, the empirical risk takes the form
\[
L(\theta) = \frac{1}{n}\sum_{i=1}^n |f_\theta(X_{t_i}) - Y_{t_i}|^2 + \lambda \int_\Omega \mathscr{D}(f_\theta, u)^2 \, du =  \frac{1}{n}\sum_{i=1}^n |f_\theta(X_{t_i}) - Y_{t_i}|^2 + \|\sqrt{\lambda}M\theta\|_2^2,
\]
where $\lambda > 0 $ is a hyperparameter.
From \eqref{eq:weakl2} it follows that the minimizer of the empirical risk is
$\hat \theta = (\mathbb \Phi^\ast \mathbb \Phi+nM)^{-1} \mathbb\Phi^\ast \mathbb Y$. It is shown in \citet{doumeche2024physicsinformed} that, as $n \to \infty$, $f_{\hat{\theta}}$ converges to $f^\star$ under appropriate assumptions. Moreover, incorporating the differential constraint improves the learning process; in particular, $f_{\hat{\theta}}$ converges to $f^\star$ faster when $\lambda > 0$.

\section{A toy-example of hierarchical forecasting}
\label{sec:toy-example}
\paragraph{Setting.} We evaluate the performance of WeaKL on a simple but illustrative hierarchical forecasting task. In this simplified setting, we want to forecast two random variables, $Y_1$ and $Y_2$, defined as follows:
\[Y_1 = \langle X_1, \theta_1 \rangle + \varepsilon_1, \quad Y_2 = \langle X_2, \theta_2 \rangle - \varepsilon_1 + \varepsilon_2,
\]
where $X_1$, $X_2$, $\varepsilon_1$, and $\varepsilon_2$ are independent. The feature vectors are $X_1 \sim \mathcal{N}(0, \mathrm{I}_d)$ and $X_2 \sim \mathcal{N}(0, \mathrm{I}_d)$, with $d \in \mathbb N^\star$. The noise terms follow Gaussian distributions $\varepsilon_1 \sim \mathcal{N}(0, \sigma_1^2)$ and $\varepsilon_2 \sim \mathcal{N}(0, \sigma_2^2)$, with $\sigma_1, \sigma_2 > 0$.
Note that the independence assumption aims at simplifying the analysis in this toy-example by putting the emphasis on the impact of the hierarchical constraints rather than on the autocorrelation of the signal, though in practice this assumption is unrealistic for most time series. 
This is why we will develop a use case of hierarchical forecasting with real-world time series in Section~\ref{sec:tourism}.

What distinguishes this hierarchical prediction setting is the assumption that $\sigma_1 \geq \sigma_2$. 
Consequently, conditional on $X_1$ and $X_2$, the sum $Y_1 + Y_2= \langle  X_1, \theta_1 \rangle + \langle  X_2, \theta_2 \rangle + \varepsilon_2$ has a lower variance than either $Y_1$ or $Y_2$. 
We assume access to $n$ i.i.d.~copies $(X_{1,i}, X_{2,i}, Y_{1,i}, Y_{2,i})_{i=1}^n$ of the random variables $(X_1, X_2, Y_1, Y_2)$. 
The goal is to construct three estimators $\hat{Y}_1$, $\hat{Y}_2$, and $\hat{Y}_3$ of $Y_1$, $Y_2$, and $Y_3:=Y_1+Y_2$.

\paragraph{Benchmark.} We compare four techniques. The \textit{bottom-up (BU)} approach involves running two separate ordinary least squares (OLS) regressions that independently estimate $Y_1$ and $Y_2$ without using information about $Y_1 + Y_2$. Specifically,
\[
\hat{Y}_1^{\mathrm{BU}} = \langle X_1, \hat{\theta}_1^{\mathrm{BU}} \rangle, \quad \hat{Y}_2^{\mathrm{BU}} = \langle X_2, \hat{\theta}_2^{\mathrm{BU}} \rangle,
\]
where the OLS estimators are
\[
\hat{\theta}_1^{\mathrm{BU}} = (\mathbb{X}_1^\top \mathbb{X}_1)^{-1} \mathbb{X}_1^\top \mathbb{Y}_1, \quad \hat{\theta}_2^{\mathrm{BU}} = (\mathbb{X}_2^\top \mathbb{X}_2)^{-1} \mathbb{X}_2^\top \mathbb{Y}_2.
\]
Here, $\mathbb X_1 = (X_{1,1} \mid  \cdots \mid 
    X_{1,n})^\top$ and  $\mathbb X_2 = (X_{2,1}\mid \cdots \mid
    X_{2,n})^\top$ are $n \times d$ matrices, while  $\mathbb Y_1 = (Y_{1,1}, \hdots ,
    Y_{1,n})^\top$ and  $\mathbb Y_2 = (Y_{2,1}, \hdots ,
    Y_{2,n})^\top$ are vectors of $\mathbb R^n$.
To estimate $Y_3$, we simply set $\hat Y_3^{\mathrm{BU}}  = \hat Y_1^{\mathrm{BU}}  + \hat Y_2^{\mathrm{BU}} $.

The \textit{Reconciliation (Rec)} approach involves running three independent forecasts of $Y_1$, $Y_2$, and $Y_3$, followed by using the constraint that the updated estimator $\hat Y_3^{\mathrm{Rec}}$ should be the sum of $\hat Y_1^{\mathrm{Rec}}$ and $\hat Y_2^{\mathrm{Rec}}$ \citep{Wickramasuriya2019optimal}. To estimate $Y_3$, we run an OLS regression with 
$\mathbb X = (\mathbb X_1\mid  \mathbb X_2)$  and $\mathbb Y = \mathbb Y_1 +  \mathbb Y_2$. In this approach, 
\[
\begin{pmatrix}
    \hat Y_{3,t}^{\mathrm{Rec}}\\
    \hat Y_{1,t}^{\mathrm{Rec}}\\
    \hat Y_{2,t}^{\mathrm{Rec}}
\end{pmatrix} = S (S^T S)^{-1} S^T\begin{pmatrix}
    &\langle X_{t}, (\mathbb X^\top \mathbb X)^{-1}\mathbb X^\top \mathbb Y\rangle\\
    &\langle X_{1,t}, \hat \theta_1^{\mathrm{BU}}\rangle\\
    &\langle X_{2,t}, \hat \theta_2^{\mathrm{BU}}\rangle
\end{pmatrix},
\]
with $S = \begin{pmatrix}
    1 & 1 \\
    1 & 0 \\
    0 & 1
\end{pmatrix}$ and $X_t = (X_{1,t} \mid X_{2,t})$.

The \textit{Minimum Trace (MinT)} approach is an alternative update method that replaces the update matrix $S (S^\top S)^{-1} S^T$ with
$S (J-JWU(U^\top W U)^{-1}U^\top)$,
$J = \begin{pmatrix}
    0 & 1 & 0\\
    0 & 0 & 1 
\end{pmatrix}$,
$W$ the $3 \times 3$ covariance matrix of the prediction errors on the training data, and
$U = \begin{pmatrix}
    -1 &1 & 1
\end{pmatrix}^\top$ \citep{Wickramasuriya2019optimal}.
This approach extends the linear projection onto $\mathrm{Im}(S)$ and better accounts for correlations in the noise of the time series. 
Finally, we apply the WeaKL-BU estimator \eqref{eq:pikl-bu} with $M=0$, $\phi_1(x) = x$, $\phi_2(x) = x$, and $\Lambda = \mathrm{Diag}(1,1, \lambda)$, where $\lambda > 0$ is a hyperparameter that controls the penalty on the joint prediction $Y_{1} + Y_{2}$. 
It minimizes the empirical loss
\[
L(\theta_1, \theta_2) = \frac{1}{n}\sum_{i=1}^n |\langle X_{1,i}, \theta_1\rangle- Y_{1,i}|^2 + |\langle X_{2,i}, \theta_2\rangle-Y_{2,i}|^2 + \lambda |\langle X_{1,i}, \theta_1\rangle + \langle X_{1,i}, \theta_2\rangle- Y_{1,i}-Y_{2,i}|^2,
\]
In the experiments, we set $\lambda = \sigma_2^{-2}$ for simplicity, although it can be easily learned by cross-validation.
\begin{figure}
    \centering
    \includegraphics[width=0.4\linewidth]{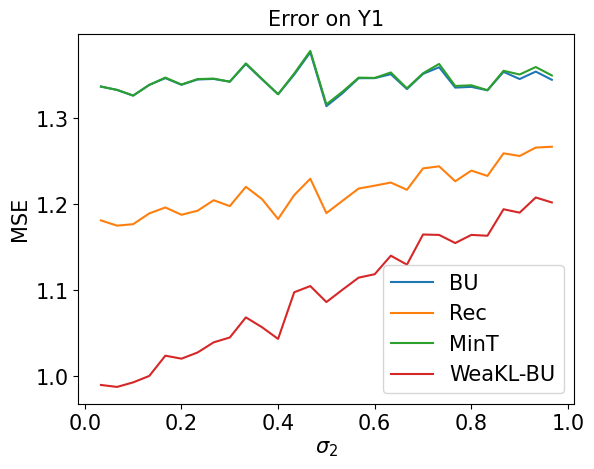}
    \includegraphics[width=0.4\linewidth]{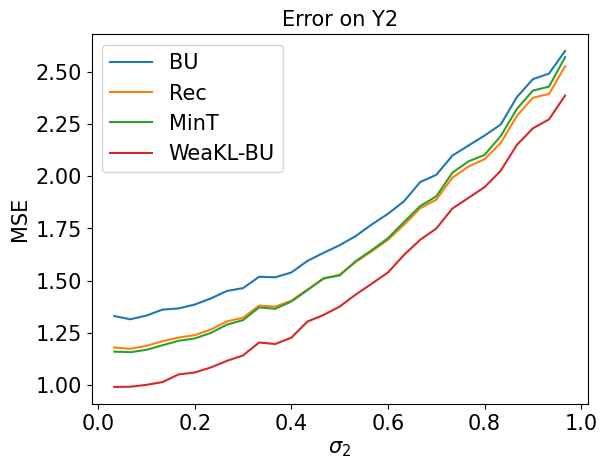}
    \includegraphics[width=0.4\linewidth]{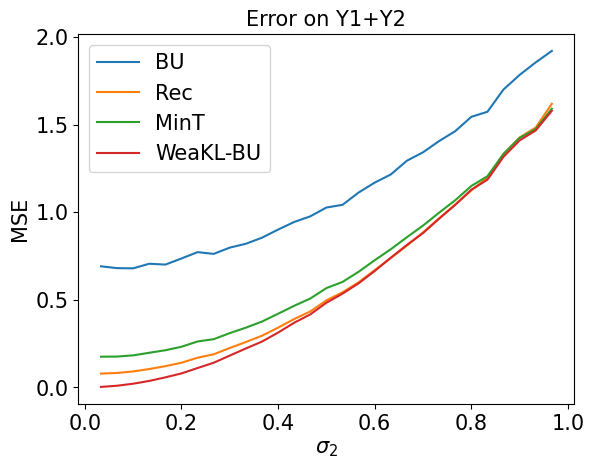}
    \includegraphics[width=0.4\linewidth]{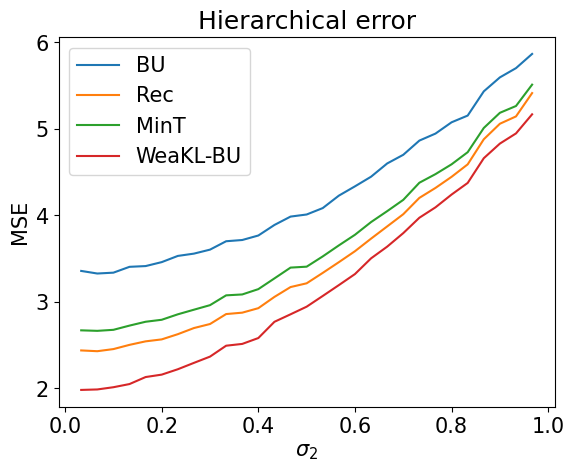}
    \caption{Hierarchical forecasting performance with $2d/n = 0.5$.}
    \label{fig:hier}
\end{figure}

\paragraph{Monte Carlo experiment.} To compare the performance of the different methods, we perform a Monte Carlo experiment. Since linear regression is invariant under multiplication by a constant, we set $\sigma_1=1$ without loss of generality. 
Since $\sigma_2 \leq \sigma_1$, we allow $\sigma_2$ to vary from $0$ to $1$. For each value of $\sigma_2$, we run $1000$ Monte Carlo simulations, where each simulation uses $n = 80$ training samples and $\ell = 20$ test samples. In each Monte Carlo run, we independently draw $\theta_1 \sim \mathcal N(0, I_d)$, $\theta_2 \sim \mathcal N(0, I_d)$, $X_{1,i} \sim \mathcal N(0, I_{ d})$, $X_{2,i} \sim \mathcal N(0, I_{d})$, $\varepsilon_{1,i} \sim \mathcal N(0, 1)$, and $\varepsilon_{2,i} \sim \mathcal N(0, \sigma_2^2)$, where $1 \leq i \leq n$. 
Note that, on the one hand, the $L^2$ error of an OLS regression on $Y_1 + Y_2$ is $\sigma_2^2 (1 + 2d/n)$, while on the other hand, the minimum possible $L^2$ error when fitting  $Y_1 + Y_2$ is $\sigma_2^2$.
Thus,  a large $2d/n$ is necessary to observe the benefits of hierarchical prediction. To achieve this, we set $d = 20$, resulting in $2d/n = 0.5$.

The models are trained on the $n$ training data points, and their performance is evaluated on the $\ell$ test data points using the mean squared error (MSE). Given any estimator $(\hat{Y}_1$, $\hat{Y}_2$, $\hat{Y}_3)$ of $(Y_1, Y_2, Y_1+Y_2)$, we compute the error $\ell^{-1}\sum_{j=1}^\ell| Y_{1, n+j}- \hat Y_{1, n+j}|^2$ on $Y_1$, the error $\ell^{-1}\sum_{j=1}^\ell| Y_{2, n+j}- \hat Y_{2, n+j}|^2$ on $Y_2$, and the error $\ell^{-1}\sum_{j=1}^\ell| Y_{1, n+j} + Y_{2, n+j}- \hat Y_{3, n+j}|^2$ on $Y_1 + Y_2$. 
The hierarchical error is defined as the sum of these three MSEs, which are visualized in Figure \ref{fig:hier}. 

\paragraph{Results.} Figure \ref{fig:hier} clearly shows that all hierarchical models (Rec, MinT, and WeaKL) outperform the naive bottom-up model for all four MSE metrics. Among them, our WeaKL consistently emerges as the best performing model, achieving superior results for all values of $\sigma_2$. Our WeaKL delivers gains ranging from $10\%$ to $50\%$ over the bottom-up model, solidifying its effectiveness in the benchmark.

The strong performance of WeaKL can be attributed to its approach, which goes beyond simply computing the best linear combination of linear experts to minimize the hierarchical loss, as reconciliation methods typically do. Instead, WeaKL directly optimizes the weights $\theta_1$ and $\theta_2$ to minimize the hierarchical loss.
Another way to interpret this is that when the initial forecasts are suboptimal, reconciliation methods aim to find a better combination of those forecasts, but do so without adjusting their underlying weights. In contrast, the WeaKL approach explicitly recalibrates these weights, resulting in a more accurate and adaptive hierarchical forecast.

\paragraph{Extension to the over-parameterized limit.} Another advantage of WeakL is that it also works for $d$ such that $2n \geq 2d \geq n$. 
In this context, the Rec and MinT algorithms cannot be computed because the OLS regression of $\mathbb Y$ on $\mathbb X$ is overparameterized ($2d$ features but only $n$ data points). 
To study the performance of the benchmark in the $n \simeq d$ limit, we repeated the same Monte Carlo experiment, but with $d = 38$, resulting in $d/n = 0.95$.
\begin{figure}
    \centering
    \includegraphics[width=0.4\linewidth]{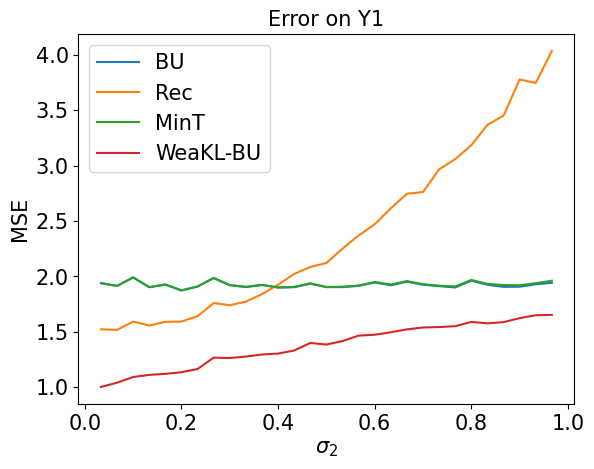}
    \includegraphics[width=0.4\linewidth]{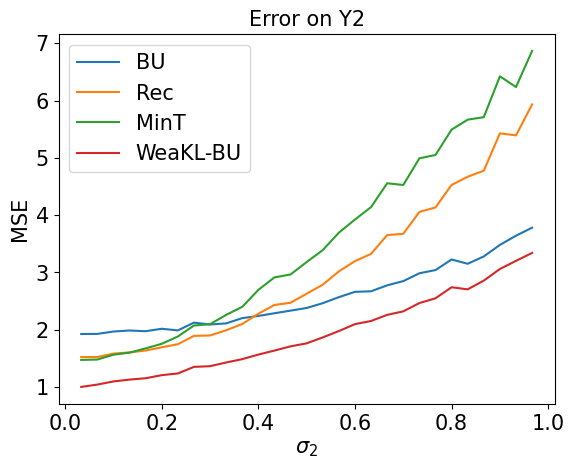}
    \includegraphics[width=0.4\linewidth]{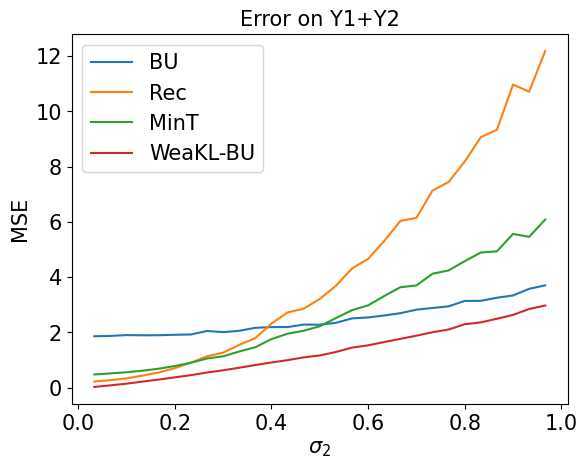}
    \includegraphics[width=0.4\linewidth]{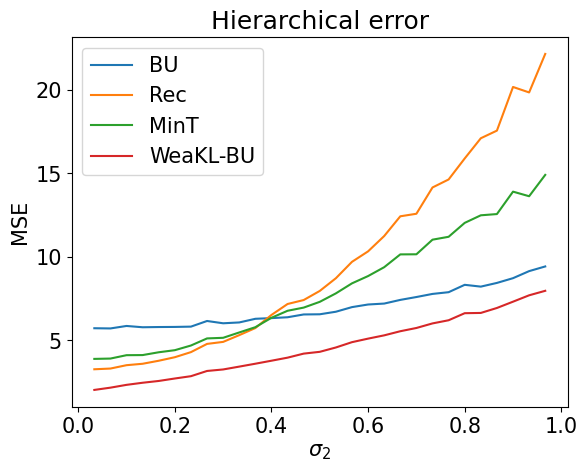}
    \caption{Hierarchical forecasting performance with $2d/n = 0.95$.}
    \label{fig:hier_2}
\end{figure}
The MSEs of the methods are shown in Figure~\ref{fig:hier_2}. These results further confirm the superiority of the WeaKL approach in the overparameterized regime. 
Note that such overparameterized situations are common in hierarchical forecasting.  
For example, forecasting an aggregate index-such as electricity demand, tourism, or food consumption-at the national level using city-level data across $d \gg 1$ cities (e.g., local temperatures) often leads to an overparameterized model.

\paragraph{Extension to non-linear regressions.} For simplicity, our experiments have focused on linear regressions. However, it is important to note that the hierarchical WeaKL can be applied to nonlinear regressions using exactly the same formulas. Specifically, in cases where $Y_1 = f_1(X_1) + \varepsilon_1$ and $Y_2 = f_2(X_2) - \varepsilon_1 + \varepsilon_2$, where $f_1$ and $f_2$ represent nonlinear functions, the WeaKL approach remains valid. This is because the connection to the linear case is straightforward: the WeaKL essentially performs a linear regression on the Fourier coefficients of $X_1$ and $X_2$, seamlessly extending its applicability to nonlinear settings.

\section{Experiments}
This appendix provides comprehensive details on the use cases discussed in the main text.
Appendix~\ref{sec:tuning} describes our hyperparameter tuning technique.
Appendix~\ref{sec:block-bootstrap} explains how we evaluate uncertainties.
Appendix~\ref{sec:half-hour} outlines our approach to handling sampling frequency in electricity demand forecasting applications.
Appendix~\ref{sec:case_study1} details the models used in Use case 1, while Appendix~\ref{sec:sobriety} focuses on Use case $2$, and Appendix~\ref{sec:hierarchical_details} covers the tourism demand forecasting use case.

\subsection{Hyperparameter tuning}
\label{sec:tuning}
\paragraph{Hyperparameter tuning of the additive WeaKL.} Consider a WeaKL additive model
\begin{align*}
    f_{\theta}(X_t) &=  \langle \phi_{1,1}(X_{1,t}), \theta_{1,1}\rangle + \cdots + \langle \phi_{1,d_1}(X_{d_1,t}), \theta_{1,d_1}\rangle,
\end{align*}
where the type (linear, nonlinear, or categorical) of the effects are specified. Thus, as detailed in Section~\ref{sec:shape},
\begin{itemize}
    \item[$(i)$] If the effect $\langle \phi_{1,\ell}(X_{\ell,t}), \theta_{1,j}\rangle$ is assumed to be linear, then $\phi_{1,j}(X_{\ell,t}) = X_{\ell,t}$,
    \item[$(ii)$] If the effect $\langle \phi_{1,\ell}(X_{\ell,t}), \theta_{1,\ell}\rangle$ is assumed to be  nonlinear, then $\phi_{1,\ell}$ is a Fourier map with $2m_\ell +1$ Fourier modes,
    \item[$(iii)$] If the effect $\langle \phi_{1,\ell}(X_{\ell,t}), \theta_{1,\ell}\rangle$ is assumed to be categorical with values in $E$, then $\phi_{1,\ell}$ is a Fourier map with $2\lfloor |E|/2\rfloor+1$ Fourier modes.
\end{itemize}
We let $\mathbf{m} = \{m_\ell\mid \hbox{the effect $\langle \phi_{1,\ell}(X_{\ell,t}), \theta_{1,\ell}\rangle$ is nonlinear}\}$ be the concatenation of the numbers of Fourier modes of the nonlinear effects.
The goal of hyperparameter tuning is to find the best set of hyperparameters $\lambda = (\lambda_1, \hdots, \lambda_{d_1})$ and 
$\mathbf{m}$ for the empirical risk \eqref{eq:weaklGAM} of the additive WeaKL.

To do so, we split the data into three sets: a training set, then a validation set, and finally a test set.
These three sets must be disjoint to avoid overfitting, and the test set is the dataset on which the final performance of the method will be evaluated.
The sets should be chosen so that the distribution of $(X,Y)$ on the validation set resembles as much as possible the distribution of $(X,Y)$ on the test set.

We consider a list of potential candidates for the optimal set of hyperparameters $(\lambda ,\mathbf{m})_{\mathrm{opt}}$. 
Since we have no prior knowledge about $(\lambda, \mathbf{m})$, we chose this list to be a grid of parameters. For each element $(\lambda, \mathbf{m})$ in the grid, we compute the minimizer $\hat \theta(\lambda, \mathbf{m})$ of the loss \eqref{eq:weaklGAM} over the training period. Then, given $\hat \theta(\lambda, \mathbf{m})$, we compute the mean squared error (MSE) of $f_{\hat \theta(\lambda, \mathbf{m})}$ over the validation period. This procedure is commonly referred to as grid search. The resulting estimate of the optimal hyperparameters $(\lambda, \mathbf{m})_{\mathrm{opt}}$ corresponds to the values of  $(\lambda, \mathbf{m})$ that minimize the MSE of  $f_{\hat \theta(\lambda, \mathbf{m})}$ over the validation period. The performance of the additive WeaKL is then assessed based on the performance of $f_{\hat \theta(\lambda, \mathbf{m})_{\mathrm{opt}}}$ on the test set.
\paragraph{Hyperparameter tuning of the online WeaKL.} Consider an online WeaKL
\begin{equation*}
    f_\theta(t, x_1, \hdots, x_{d_1}) = h_{\theta_0}(t)+ \sum_{\ell=1}^{d_1} (1+ h_{\theta_\ell}(t))  \hat g_\ell(x_\ell),
\end{equation*}
where the effects $\hat g_\ell$ are known, and the updates $h_{\theta_\ell}(t) = \langle \phi(t), \theta_\ell\rangle$ are such that $\phi$ is the Fourier map $\phi(t) =(\exp(i k t/2))_{-m_j\leq k \leq m_j}^\top$, with $m_j \in \mathbb N^\star$.
We let $\mathbf{m} = \{m_j\mid 0\leq j \leq d_1\}$ be the concatenation of the numbers of Fourier modes.
The goal of hyperparameter tuning is to find the best set of hyperparameters $\lambda = (\lambda_0, \hdots, \lambda_{d_1})$ and 
$\mathbf{m}$ for the empirical risk \eqref{eq:risk_online} of the online WeaKL.

To do so, we split the data into three sets: a training set, then a validation set, and finally a test set.
This three sets must be disjoint to avoid overfitting.
Moreover, the training set and the validation set must be disjoint from the data used to learn the effects $\hat g_\ell$.
The test set must be the set on which the final performance of the method will be evaluated. 
The sets should be chosen so that the distribution of $(X,Y)$ on the validation set resembles as much as possible the distribution of $(X,Y)$ on the test set.
Similarly to the hyperparameter tuning of the additive WeaKL, we then consider a list of potential candidates for the optimal hyperparameter $(\lambda ,\mathbf{m})_{\mathrm{opt}}$, which can be a grid.
Then, we compute the minimizer $\hat \theta(\lambda, \mathbf{m})$ of the loss \eqref{eq:risk_online} on the training period, and the resulting estimation of $(\lambda, \mathbf{m})_{\mathrm{opt}}$ is the set of hyperparameters $(\lambda, \mathbf{m})$ such that the MSE of $f_{\hat \theta(\lambda, \mathbf{m})}$ on the validation period is minimal.
The performance of the online WeaKL is thus measured by the performance of $f_{\hat \theta(\lambda, \mathbf{m})_{\mathrm{opt}}}$ on the test set.

\subsection{Block bootstrap methods}
\label{sec:block-bootstrap}
\paragraph{Evaluating uncertainties with block bootstrap.}
The purpose of this paragraph is to provide theoretical tools for evaluating the performance of time series estimators.
Formally, given a test period $\{t_1, \hdots, t_n\}$, a  target time series $(Y_{t_j})_{1\leq j \leq n}$, and an estimator $(\hat Y_{t_j})_{1\leq j \leq n}$ of $Y$, the goal is to construct confidence intervals that quantify how far $\mathrm{RMSE}_{n} = (n^{-1}\sum_{j=1}^n |\hat Y_{t_j} - Y_{t_j}|^2)^{1/2}$ deviates from its expectation 
$\mathrm{RMSE} = (\mathbb{E} |\hat Y_{t_1} - Y_{t_1}|^2)^{1/2}$, and  how far  $\mathrm{MAPE}_n = n^{-1}\sum_{j=1}^n |\hat Y_{t_j} - Y_{t_j}| |Y_{t_j}|^{-1}$ deviates from its expectation $\mathrm{MAPE} = \mathbb{E}( |\hat Y_{t_1} - Y_{t_1}| |Y_{t_1}|^{-1})$.
Here, we assume that $Y$ and $\hat{Y}$ are strongly stationary, meaning their distributions remain constant over time.
Constructing such confidence intervals is non-trivial because the observations $Y_{t_j}$ in the time series $Y$ are correlated, preventing the direct application of the central limit theorem.
The block bootstrap algorithm is specifically designed to address this challenge and is defined as follows.

Consider a sequence $Z_{t_1},Z_{t_2},\hdots,Z_{t_n}$ such that the quantity of interest can be expressed as $g(\mathbb{E}(Z_{t_1}))$, for some function $g$. 
This quantity is estimated by $g(\bar Z_n)$, where $\bar Z_n = n^{-1}\sum_{j=1}^n Z_{t_j}$ is the empirical mean of the sequence.
For example, $\mathrm{RMSE} = g(\mathbb E(Z_{t_1}))$ and $\mathrm{RMSE}_n = g(\bar Z_n)$ for $g(x) = x^{1/2}$ and $Z_{t_j}=(Y_{t_j}-\hat{Y_{t_j}})^2$, while $\mathrm{MAPE} = g(\mathbb E(Z_{t_1}))$ and  $\mathrm{MAPE}_n = g(\bar Z_n)$  for $g(x) = x$ and $Z_{t_j}=|\hat Y_{t_j} - Y_{t_j}| |Y_{t_j}|^{-1}$.
The goal of the block bootstrap algorithm is to estimate the distribution of $g(\bar Z_n)$.

Given a length $\ell \in \mathbb N^\star$ and a starting time $t_j$, we say that $(Z_{t_j}, \hdots, Z_{t_{j+\ell-1}})\in \mathbb R^\ell$ is a block of length $\ell$ starting at $t_j$. We draw  $b = \lfloor n/\ell\rfloor +1$ blocks of length $\ell$ uniformly from the sequence  $(Z_{t_1}, Z_{t_2}, \dots, Z_{t_n})$ and then concatenate these blocks to form the sequence $Z^\ast = (Z_1^\ast, Z_2^\ast, \dots, Z_{b\ell}^\ast)$.
Thus, $Z^\ast$ is a resampled version of $Z$  obtained with replacement.

For convenience, we consider only the first $n$ values of $Z^\ast$ and compute the bootstrap version of the empirical mean: $\bar{Z}^\ast_n=\frac{1}{n}\sum_{j=1}^nZ^\ast_j$. By repeatedly resampling the $b$ blocks and generating multiple instances of $\bar{Z}^\ast_n$, the resulting distribution of $\bar{Z}^\ast_n$  provides a reliable estimate of the distribution of $\bar{Z}_n$.
In particular, under general assumptions about the decay of the autocovariance function of $Z$, choosing $\ell = \lfloor n^{1/4} \rfloor$ leads to 
\[\sup_{x\in\mathbb{R}}|\mathbb{P}(T^{\ast}_n\leq x\mid Z_{t_1},\hdots,Z_{t_n})-\mathbb{P}(T_n\leq x)| = O_{n\to \infty}(n^{-3/4}),\]  where $T^{\ast}_n=\sqrt n(\bar{Z}^\ast_n-\mathbb{E}(\bar{Z}^\ast_n\mid Z_{t_1},\hdots,Z_{t_n}))$ and $T_n=\sqrt n(\bar{Z}_n-\mathbb{E}(Z_{t_1}))$ \citep[see, e.g.][Theorem 6.7]{lahiri2013resampling}.
Note that this convergence rate of $n^{-3/4}$ is actually quite fast, as even if the $Z_{t_j}$  were i.i.d., the empirical mean $\bar{Z}_n$  would only converge to a normal distribution at a rate of $n^{-1/2}$
(by the Berry-Esseen theorem). This implies that the block bootstrap method estimates the distribution of $\bar{Z}_n$ faster than $\bar{Z}_n$  itself converges to its Gaussian limit.

The choice of $\ell$ plays a crucial role in this method.
For instance, setting $\ell = 1$ leads to an underestimation of the variance of $\bar{Z}_n$  when the $Z_{t_j}$ are correlated \citep[see, e.g.][Corollary 2.1]{lahiri2013resampling}.
In addition, block resampling introduces a bias, as $Z_{t_n}$ belongs to only a single block and is therefore less likely to be resampled than  $Z_{t_{\lfloor n/2\rfloor}}$.
This explains why $\mathbb{E}(\bar{Z}^\ast_n \mid Z_{t_1}, \dots, Z_{t_n}) \neq \bar{Z}_n$.
To address both problems, \citet{politis1994stationary} introduced the stationary bootstrap, where the block length $\ell$ varies and follows a geometric distribution.

\paragraph{Comparing estimators with block bootstrap.}
Given two stationary estimators $\hat Y^1$ and $\hat Y^2$ of $Y$,  the goal is to develop a test of level $\alpha \in [0,1]$ for the hypothesis $H_0: \mathbb E|\hat Y^1_t-Y_t| = \mathbb E|\hat Y^2_t-Y_t|$. Using the previous paragraph, such a test could be implemented by estimating two confidence intervals $I_1$ and $I_2$ for $\mathbb E|\hat Y^1_t-Y_t|$ and $\mathbb E|\hat Y^2_t-Y_t|$ at level $\alpha/2$ using block bootstrap, and then rejecting $H_0$  if  $I_1 \cap I_2 = \emptyset$. However, this approach tends to be conservative, potentially reducing the power of the test when assessing whether one estimator is significantly better than the other.  

To create a more powerful test, \citet{Messner2020evaluation} and \citet{Farrokhabadi2022day} suggest relying on the MAE skill score, which is defined by \[\mathrm{Skill}=1-\frac{\mathrm{MAE_{1}}}{\mathrm{MAE_{2}}},\]
where $\mathrm{MAE_{1}}$ and $\mathrm{MAE_{2}}$ are the mean average errors of $\hat{Y}^{1}$ and $\hat{Y}^2$, respectively. 
Note that $\mathrm{Skill}= (\mathrm{MAE}_{2} - 
\mathrm{MAE_{1}})/\mathrm{MAE_{2}}$
is the relative distance between the two $\mathrm{MAE}$s. Thus, $\hat{Y}^{1}$ is significantly better than $\hat{Y}^2$ if $\mathrm{Skill}$ is significantly positive.
A confidence interval for $\mathrm{Skill}$ can be obtained  by block bootstrap.
Indeed, consider the time series $Z$ defined as $Z_{t_j}=(|\hat{Y}^{1}_{t_j}-Y^1_{t_j}|,|\hat{Y}^{2}_{t_j}-Y^2_{t_j}|)$, and let $g(x,y)=1-x/y$. We use the block bootstrap method over this sequence to estimate $g(\mathbb E(Z))$ by generating different samples of $\mathrm{MAE_{1}}$ and $\mathrm{MAE_{2}}$. In particular, in Appendix~\ref{sec:case_study1}, $\hat{Y}^{1}$ corresponds to WeakL, while $\hat{Y}^{2}$ is the estimator of the winning team of the IEEE competition.

\subsection{Half-hour frequency}
\label{sec:half-hour}

Short-term electricity demand forecasts are often estimated with a half-hour frequency, meaning that the objective is to predict electricity demand every $30$ minutes during the test period. This applies to both Use case 1 and Use case 2.
There are two common approaches to handling this frequency in forecasting models. One approach is to include the half-hour of the day as a feature in the models. The alternative, which yields better performance, is to train a separate model for each half-hour, resulting in $48$ distinct models. This superiority arises because the relationship between electricity demand and conventional features (such as temperature and calendar effects) varies significantly across different times of the day. For instance, electricity demand remains stable at night but fluctuates considerably during the day. This variability justifies treating the forecasting problem at each half-hour as an independent learning task, leading to $48$ separate models. Consequently, in both use cases, all models discussed in this paper---including WeaKL, as well as those from \citet{vilmarest2022state} and \citet{doumeche2023human}---are trained separately for each of the $48$ half-hours, using identical formulas and architectures. This results in $48$ distinct sets of model weights. 
For simplicity, and since the only consequence of this preprocessing step is to split the learning data into $48$ independent groups, this distinction is omitted from the equations.

\subsection{Precisions on the Use case 1 on the IEEE DataPort Competition on Day-Ahead Electricity Load Forecasting}
\label{sec:case_study1}
In this appendix, we provide additional details on the two WeaKLs used in the benchmark for Use case 1 of the IEEE DataPort Competition on Day-Ahead Electricity Load Forecasting. The first model is a direct adaptation of the GAM-Viking model from \citet{vilmarest2022state} into the WeaKL framework. The second model is a WeaKL where the effects are learned through hyperparameter tuning.
\paragraph{Direct translation of the GAM-Viking model into the WeaKL framework.}

To build their model, \citet{vilmarest2022state} consider four primary models: an autoregressive model (AR), a linear regression model, a GAM, and a multi-layer perceptron (MLP). 
These models are initially trained on data from $18$ March $2017$ to $1$ January $2020$. 
Their weights are then updated using the Viking algorithm starting from $1$ March $2020$ \citep[][Table~3]{vilmarest2022state}. 
The parameters of the Viking algorithm were manually selected by the authors based on performance monitoring over the $2020$–$2021$ period \citep[][Figure7]{vilmarest2022state}. 
To further refine the forecasts, the model errors are corrected using an autoregressive model, which they called the intraday correction and implemented as a static Kalman filter. The final forecast is obtained by online aggregation of all models, meaning that the predictions from different models are combined in a linear combination that evolves over time. The weights of this aggregation are learned using the ML-Poly algorithm from the \texttt{opera} package \citep{gaillard2016opera}, trained over the period $1$ July $2020$ to $18$ January $2021$. 
The test period spans from $18$ January $2021$ to $17$ February $2021$. 
During this period, the aggregated model achieves a MAE of $10.9$~GW, while the Viking-updated GAM model alone yields an MAE of $12.7$~GW.

Here, to ensure a fair comparison between our WeaKL framework and the GAM-Viking model of \citet{vilmarest2022state}, we replace their GAM-Viking with our online WeaKL in their aggregation. Our additive WeaKL model is therefore a direct translation of their offline GAM formulation into the WeaKL framework. Specifically, we consider the additive WeaKL based on the features $X = (\mathrm{DoW}, \mathrm{FTemps95_{corr1}}, \mathrm{Load_1}, \mathrm{Load_7}, \mathrm{ToY}, t)$ corresponding to
\begin{equation*}
\begin{split} 
Y_t =& g_1^\star(\mathrm{DoW}_t)
   +g_2^\star(\mathrm{FTemps95_{corr1}}_t)
   + g_3^\star(\mathrm{Load_1}_t) +g_4^\star(\mathrm{Load_7}_t)+g_5^\star(\mathrm{ToY}_t)
      +g_6^\star(t) +\varepsilon_t,
\end{split}
\end{equation*}
where $g_1^\star$ is categorical with 7 values, $g_2^\star$ and $g_6^\star$  are linear,  $g_3^\star$, $g_4^\star$, and $g_5^\star$ are nonlinear.

$\mathrm{FTemps95_{corr1}}$ is a smoothed version of the temperature, while the other features remain the same as those used in Use case 2. The weights of the additive WeaKL model are determined using the hyperparameter selection technique described in Appendix~\ref{sec:tuning}. 
The training period spans from $18$ March $2017$ to $1$ November $2019$, while the validation period extends from $1$ November $2019$ to $1$ January $2020$. 
During this grid search, the performance of $250,047$ sets of hyperparameters $(\lambda, \mathbf m)\in \mathbb R^7\times \mathbb R^3$ is evaluated in less than a minute using a standard GPU (Nvidia L4 GPU, $24$~GB RAM, $30.3$ teraFLOPs for Float32). 
Notably, this optimization period exactly matches the training period of the primary models in \citet{vilmarest2022state}, ensuring a fair comparison between the two approaches.

Then, we run an online WeaKL, where the effects $\hat g_\ell$, $1\leq \ell \leq 7$, are inherited directly from the previously trained additive WeaKL. 
The weights of this online WeaKL are determined using the hyperparameter selection technique described in Appendix~\ref{sec:tuning}. 
The training period extends from $1$ February $2020$ to $18$ November $2020$, while the validation period extends from $18$ November $2020$ to $18$ January $2021$, immediately preceding the final test period to ensure optimal adaptation. During this grid search, we evaluate $625$ sets of hyperparameters $(\lambda, \mathbf m)\in \mathbb R^6\times \mathbb R^6$ in less than a minute using a standard GPU. 
Since $t$ is already included as a feature, the function  $h_0^\ast$ in Equation~\eqref{eq:model} is not required in this setting.

Finally, we evaluate the performance of our additive WeaKL (denoted as $\hbox{WeaKL}_{\mathrm{+}}$), our additive WeaKL followed by intraday correction ($\hbox{WeaKL}_{+,\mathrm{intra}}$), our online WeaKL ($\hbox{WeaKL}_{\mathrm{on}}$), our online WeaKL with intraday correction ($\hbox{WeaKL}_{\mathrm{on, intra}}$), and an aggregated model based on \citet{vilmarest2022state}, where the GAM and GAM-Viking models are replaced by our additive and online WeaKL models ($\hbox{WeaKL}_{\mathrm{agg}}$). The test period remains consistent with \citet{vilmarest2022state}, spanning from $18$ January $2021$ to $17$ February $2021$. Their performance results are presented in Table~\ref{tab:gam-vik} and compared to their corresponding translations within the GAM-Viking framework.
\begin{table}
    \centering
    \caption{Comparing GAM-Viking with its direct translation in the WeaKL framework on the final test period}
    \begin{tabular}{|c|ccccc|}
        \hline
        Model GAM &  $\hbox{GAM}_{\mathrm{+}}$ & $\hbox{GAM}_{+,\mathrm{intra}}$ & $\hbox{GAM}_{\mathrm{on}}$ & $\hbox{GAM}_{\mathrm{on, intra}}$ & $\hbox{GAM}_{\mathrm{agg}}$\\
        \hline
         MAE (GW) & 48.3 & 22.7 & 13.2 & 12.7 & 10.9\\
         \hline
         \hline
        Model WeaKL &  $\hbox{WeaKL}_{\mathrm{+}}$ & $\hbox{WeaKL}_{+,\mathrm{intra}}$ & $\hbox{WeaKL}_{\mathrm{on}}$ & $\hbox{WeaKL}_{\mathrm{on, intra}}$ & $\hbox{WeaKL}_{\mathrm{agg}}$\\
        \hline
         MAE (GW) & 58.0 & 23.4 & 11.2 & 11.3 & 10.5\\
         \hline
    \end{tabular}
    \label{tab:gam-vik}
\end{table}
Thus, $\hbox{GAM}_{\mathrm{+}}$ refers to the offline GAM, while $\hbox{GAM}_{+,\mathrm{intra}}$ corresponds to the offline GAM with an intraday correction. Similarly, $\hbox{GAM}_{\mathrm{on}}$ represents the GAM-Viking model, and $\hbox{GAM}_{\mathrm{on, intra}}$ denotes the GAM-Viking model with an intraday correction. Finally, $\hbox{GAM}_{\mathrm{agg}}$ corresponds to the final model proposed by \citet{vilmarest2022state}.

The performance $\hbox{GAM}_{\mathrm{+}}$, $\hbox{GAM}_{\mathrm{+, intra}}$, $\hbox{WeaKL}_{\mathrm{+}}$, and $\hbox{WeaKL}_{\mathrm{+, intra}}$ in Table~\ref{tab:gam-vik} alone is not very meaningful because the distribution of electricity demand differs between the training and test periods. To address this, Table~\ref{tab:gam-vik-norm} presents a comparison of the same algorithms, trained on the same period but evaluated on a test period spanning from $1$ January $2020$ to $1$ March $2020$. In this stationary period, WeaKL outperforms the GAMs.
\begin{table}
    \centering
    \caption{Comparing GAM with its direct translation in the WeaKL framework on a stationary test period.}
    \begin{tabular}{|c|cccc|}
        \hline
        Model &  $\hbox{GAM}_{\mathrm{+}}$ &  $\hbox{WeaKL}_{\mathrm{+}}$ & $\hbox{GAM}_{+,\mathrm{intra}}$ & $\hbox{WeaKL}_{+,\mathrm{intra}}$\\
        \hline
         MAE (GW) & 20.7 & 19.1 & 19.3 & 19.2\\
         \hline
    \end{tabular}
    \label{tab:gam-vik-norm}
\end{table}

Moreover, in Table~\ref{tab:gam-vik}, the online WeaKLs clearly outperform the GAM-Viking models, achieving a reduction in MAE of more than $10\%$.
As a result, replacing the GAM-Viking models in the aggregation leads to improved overall performance. Notably, the WeaKLs are direct translations of the GAM-Viking models, meaning that the performance gains are due solely to model optimization and not to any structural changes.

\paragraph{Pure WeaKL.}
In addition, we trained an additive WeaKL using a different set of variables than those in the GAM model, aiming to identify an optimal configuration. Specifically, we consider the additive WeaKL with 
\begin{align*}
    X= (&\mathrm{FcloudCover\_corr1},\mathrm{Load1D},\mathrm{Load1W},\mathrm{DayType},\mathrm{FTemperature\_corr1},\\  &\mathrm{FWindDirection}, \mathrm{FTemps95\_corr1},\mathrm{Toy},\mathrm{t}),
\end{align*} where 
\begin{itemize}
    \item[$(i)$] the effects of $\mathrm{FclouCover\_corr1}$, $\mathrm{Load1D}$, and $\mathrm{Load1W}$ are nonlinear,
    \item[$(ii)$] the effect of $\mathrm{DayType}$ is categorical with 7 values,
    \item[$(iii)$] the remaining effects are linear.
\end{itemize}
This model is trained using the hyperparameter tuning process detailed in Appendix~\ref{sec:tuning}, with the training period spanning from $18$ March $2017$ to $1$ January $2020$, and validation starting from $1$ October $2019$.
Next, we fit an online WeaKL model, with hyperparameters tuned using a training period from $1$ March $2020$ to $18$ November $2020$ and a validation period extending until $18$ January $2021$.

To verify that our pure WeaKL model achieves a significantly lower error than the best model from the IEEE competition, we estimate the $\mathrm{MAE}$ skill score by comparing our pure WeaKL to the model proposed by \citet{vilmarest2022state}. To achieve this, we follow the procedure detailed in Appendix~\ref{sec:block-bootstrap}, using block bootstrap with a block length of $\ell = 24$ and $3000$ resamples to estimate the distribution of the $\mathrm{MAE}$ skill score, $\mathrm{Skill}$. Here, $\hat{Y}^1$ represents the WeaKL, while $\hat{Y}^2$ corresponds to the estimator from \citet{vilmarest2022state}. To evaluate the performance difference, we estimate the standard deviation $\sigma_n$ of $\mathrm{Skill}_n$ and construct an asymptotic one-sided confidence interval for $\mathrm{Skill}$. Specifically, we define
$\mathrm{Skill}_n = 1 - (\sum_{j=1}^n |\hat Y^1_{t_j} - Y_{t_j}| )/(\sum_{j=1}^n |\hat Y^2_{t_j} - Y_{t_j}|)$
and consider the confidence interval $[\mathrm{Skill}_n - 1.28 \sigma_n, +\infty[$, which corresponds to a confidence level of $\alpha = 0.1$. The resulting interval, $[0.007, +\infty[$, indicates that the $\mathrm{Skill}$ score is positive with at least $90\%$ probability. Consequently, with at least $90\%$ probability, the WeaKL  chieves a lower $\mathrm{MAE}$ than the best model from the IEEE competition.

\subsection{Precision on the use Use case 2 on forecasting the French electricity load during the energy crisis}
\label{sec:sobriety}
This appendix provides detailed information on the additive WeaKL and the online WeaKL used in Use case 2, which focuses on forecasting the French electricity load during the energy crisis.
\paragraph{Additive WeaKL.} As detailed in the main text, the additive WeaKL is built using the following features:
\[X =(\mathrm{Load}_1, \mathrm{Load}_7, \mathrm{Temp}, \mathrm{Temp}_{950},  \mathrm{Temp}_{\mathrm{max 950}}, \mathrm{Temp}_{\mathrm{min 950}}, \mathrm{ToY},  \mathrm{DoW}, \mathrm{Holiday},t).
\]

The effects of $\mathrm{Load}_1$, $\mathrm{Load}_7$, and $t$ are modeled as linear. The effects of $\mathrm{Temp}$, $\mathrm{Temp}_{950}$,  $\mathrm{Temp}_{\mathrm{max 950}}$, $ \mathrm{Temp}_{\mathrm{min 950}}$, and $\mathrm{ToY}$ are modeled as nonlinear with $m=10$. The effects of $\mathrm{DoW}$ and $\mathrm{Holiday}$ are treated as categorical, with $|E| = 7$ and $|E| = 2$, respectively. The model weights are selected through hyperparameter tuning, as detailed in Appendix~\ref{sec:tuning}. The training period spans from $8$ January $2013$ to $1$ September $2021$, while the validation period covers $1$ September $2021$ to $1$ September $2022$. Notably, this is the exact same period used by \citet{doumeche2023human} to train the GAM. The objective of the hyperparameter tuning process is to determine the optimal values for $\lambda = (\lambda_1, \hdots, \lambda_{10}) \in (\mathbb R^+)^{10}$ and $\mathbf{m} = (m_3, m_4, m_5, m_6, m_7) \in (\mathbb N^\star)^5$ in \eqref{eq:weaklGAM}. As a result, the additive WeaKL model presented in Use case 2 is the outcome of this hyperparameter tuning process.

\paragraph{Online WeaKL.} Next, we train an online WeaKL to update the effects of the additive WeaKL. To achieve this, we apply the hyperparameter selection technique detailed in Appendix~\ref{sec:tuning}. The training period spans from $1$ February $2018$ to $1$ April $2020$, while the validation period extends from $1$ April $2020$ to $1$ June $2020$. These periods, although not directly contiguous to the test period, were specifically chosen because they overlap with the COVID-19 outbreaks. This is crucial, as it allows the model to learn from a nonstationary period. Moreover, since online models require daily updates, the online WeaKL is computationally more expensive than the additive WeaKL. The training period is set to two years and two months, striking a balance between computational efficiency and GPU memory usage. Using the parameters $(\lambda, \mathbf m)$ obtained from hyperparameter tuning, we then retrain the model in an online manner with data starting from $1$ July $2020$, ensuring that the rolling training period remains at two years and two months.

\paragraph{Error quantification.} Following the approach of \citet{doumeche2023human}, the standard deviations of the errors are estimated using stationary block bootstrap with a block length of $\ell = 48$ and 1000 resamples.

\paragraph{Model running times.} Below, we present the running times of various models in the experiment that includes holidays:
\begin{itemize}
\item  GAM: $20.3$ seconds. 
\item Static Kalman adaption: $1.7$ seconds.
\item Dynamic Kalman adaption: $48$ minutes, for an hyperparameter tuning of $10^4$ sets of hyperparameters \citep[see][II.A.2]{obst2021adaptive}.
\item Viking algorithm: $215$ seconds (in addition to training the Dynamic Kalman model).
\item Aggregation: $0.8$ seconds.
\item GAM boosting model: $6.6$ seconds.
\item Random forest model: $196$ seconds.
\item Random forest + bootstrap model: $34$ seconds.
\item Additive WeaKL: grid search of $1.6\times 10^5$ hyperparameters: $257$ seconds; training a single model: $2$ seconds.
\item Online WeaKL: grid search of $9.2\times 10^3$ hyperparameters: $114$ seconds; training a single model: $52$ seconds.
\end{itemize}
\subsection{Precisions on the use case on hierarchical forecasting of Australian domestic tourism with transfer learning}
\label{sec:hierarchical_details} The matrices $\Lambda$ for the WeaKL-BU, WeaKL-G, and WeaKL-T estimators are selected through hyperparameter tuning. Following the procedure detailed in Appendix~\ref{sec:tuning}, the dataset is divided into three subsets: training, validation, and test. The training set comprises the first $60\%$ of the data, the validation set the next $20\%$, and the test set the last $20\%$. 
The optimal matrix, $\Lambda_{\mathrm{opt}}$, is chosen from a set of candidates by identifying the estimator trained on the training set that achieves the lowest MSE on the validation set. The model is then retrained using both the training and validation sets with $\Lambda = \Lambda_{\mathrm{opt}}$, and its performance is evaluated on the test set. Given that $d_1 = 415 \times 24 = 19,920$, WeaKL involves matrices of size $d_1^2 \simeq 4\times 10^8$, requiring several gigabytes of RAM. Consequently, the grid search process is computationally expensive. For instance, in this experiment, the grid search over $1024$ hyperparameter sets for WeaKL-T takes approximately $45$ minutes.